\def\year{2021}\relax
\documentclass[letterpaper]{article} 
\usepackage{aaai21}  
\usepackage{times}  
\usepackage{helvet} 
\usepackage{courier}  
\usepackage[hyphens]{url}  
\usepackage{graphicx} 
\usepackage{natbib}  
\usepackage{caption} 
\usepackage{url}            
\usepackage{booktabs}       
\usepackage{amsfonts}       
\usepackage{nicefrac}       
\usepackage{microtype}      
\usepackage{subcaption}
\usepackage{multirow}
\usepackage{rotating}

\usepackage[title]{appendix}

\usepackage{graphicx}
\usepackage{lineno}
\usepackage{amsthm}
\usepackage{amsmath}
\usepackage{bm}
\usepackage{amssymb}

\newtheorem*{rep@theorem}{\rep@title}
\newcommand{\newreptheorem}[2]{%
\newenvironment{rep#1}[1]{%
 \def\rep@title{#2 \ref{##1}}%
 \begin{rep@theorem}}%
 {\end{rep@theorem}}}
\makeatother

\newtheorem{definition}{Definition}

\newtheorem{theorem}{Theorem}
\newreptheorem{theorem}{Theorem}

\newtheorem{example}{Example}

\newtheorem{lemma}{Lemma}
\newreptheorem{lemma}{Lemma}

\newcommand{\norm}[1]{\left\lVert#1\right\rVert}

\DeclareMathOperator{\EX}{\mathbb{E}}

\newcommand{\bfe}{\mathbf{e}}

\newcommand{\bfu}{\mathbf{u}}
\newcommand{\bfs}{\mathbf{s}}

\newcommand{\bfo}{\mathbf{o}}
\newcommand{\bfr}{\mathbf{r}}

\newcommand{\bfv}{\mathbf{v}}

\newcommand{\bmw}{\bm{w}}

\newcommand{\bfnull}{\mathbf{0}}

\newcommand{\bfgamma}{\bm{\gamma}}

\DeclareMathOperator*{\argmax}{arg\,max}
\DeclareMathOperator*{\argmin}{arg\,min}

\DeclareMathOperator*{\avg}{avg}
\DeclareMathOperator*{\adj}{adj}

\usepackage[framemethod=tikz]{mdframed}

\mdfsetup{backgroundcolor=green!10}

\graphicspath{{figures/}}
\title{Explaining Neural Matrix Factorization with Gradient Rollback}

\author{
	Carolin Lawrence, Timo Sztyler, Mathias Niepert\textsuperscript{\rm 1}\\
}
\affiliations{
	\textsuperscript{\rm 1}NEC Laboratories Europe\\	
	Kurfürsten-Anlage 36\\
	D-69115 Heidelberg\\
	$\{$carolin.lawrence, timo.sztyler, mathias.niepert$\}$@neclab.eu
	
}
\begin{document}

\maketitle

\begin{abstract}
Explaining the predictions of neural black-box models is an important problem, especially when such models are used in applications where user trust is crucial. Estimating the influence of training examples on a learned neural model's behavior allows us to identify training examples most responsible for a given prediction and, therefore, to faithfully explain the output of a black-box model. The most generally applicable existing method is based on influence functions, which scale poorly for larger sample sizes and models.  

We propose \emph{gradient rollback}, a general approach for influence estimation, applicable to neural models where each parameter update step during gradient descent touches a smaller number of parameters, even if the overall number of parameters is large. Neural matrix factorization models trained with gradient descent are part of this model class. These models are popular and have found a wide range of applications in industry. Especially knowledge graph embedding methods, which belong to this class, are used extensively. We show that gradient rollback is highly efficient at both training and test time. Moreover, we show theoretically that the difference between gradient rollback's influence approximation and the true influence on a model's behavior is smaller than known bounds on the stability of stochastic gradient descent. This establishes that gradient rollback is robustly estimating example influence. We also conduct experiments which show that gradient rollback provides faithful explanations for knowledge base completion and recommender datasets. An implementation is available.\footnote{\url{https://github.com/carolinlawrence/gradient-rollback}}
\end{abstract}

\section{Introduction}
Estimating the influence a training sample (or a set of training samples) has on the behavior of a machine learning model is a problem with several useful applications. First, it can be used to interpret the behavior of the model by providing an explanation for its output in form of a set of training samples that impacted the output the most. In addition to providing a better understanding
of model behavior, influence estimation has also been used to find adversarial examples, to uncover
domain mismatch, and to determine incorrect or mislabeled examples~\cite{koh2017-influence-functions}. Finally, it can also be used to estimate the uncertainty for a particular output by exploring the stability of the output probability before and after removing a small number of influential training samples.

We propose \emph{gradient rollback} (GR), a novel approach for influence estimation. 
GR is applicable to neural models trained with gradient descent and is highly efficient especially when the number of parameters that are significantly changed during any update step is moderately sized. This is the case for neural matrix factorization methods. Here, we focus on neural link prediction models for multi-relational graphs (also known as knowledge base embedding models), as these models subsume several other matrix factorization models~\cite{guo2020survey}. They have found a wide range of industry applications such as in recommender and question answering systems. They are, however, black-box models whose predictions are not inherently interpretable. Other methods, such as rule-based methods, might be more interpretable, but
typically have worse performance. Hence, neural matrix factorization methods would greatly benefit from being more interpretable~\cite{bianchi2020knowledge}. For an illustration of matrix factorization and GR see Figure \ref{fig:overview}.

We explore two crucial questions regarding the utility of GR. First, we show that GR is highly efficient at both training and test time, even for large datasets and models. Second, we show that its influence approximation error  is smaller than known bounds on the stability of stochastic gradient descent for non-convex problems~\cite{stability-hardt}. The stability of an ML model is defined as the maximum change of its output one can expect on any sample when retrained on a slightly different set of training samples. The relationships between uniform stability and generalization of a learning system is a seminal result~\cite{bousquet2002stability}. Here, we establish a close connection between the stability of a learning system and the challenge of estimating training sample influence and, therefore, explaining the model's behavior. Intuitively, the more stable a model, the more likely is it that we can estimate the influence of training samples well. 
We show theoretically that the difference between GR's influence approximation and the true influence on the model behavior is (strictly) smaller than known bounds on the stability of stochastic gradient descent.

We perform experiments on standard matrix factorization datasets including those for knowledge base completion and recommender systems. Concretely, GR can explain a prediction of a learnt model by producing a ranked list of training examples, where each instance of the list contributed to changing the likelihood of the prediction and the list is sorted from highest to lowest impact. To produce this list, we (1) estimate the influence of training examples during training and (2) use this estimation to determine the contribution each training example made to a particular prediction. To evaluate whether GR selected training examples relevant for a particular prediction, we remove the set of training examples, retrain the model from scratch and check if the likelihood for the particular prediction decreased. Compared to baselines, GR can identify subsets of training instances that are highly influential to the model's behavior. These can be represented as a graph to explain a prediction to a user.


\begin{figure*}[t!]
	\centering
	\includegraphics[width=0.98\textwidth]{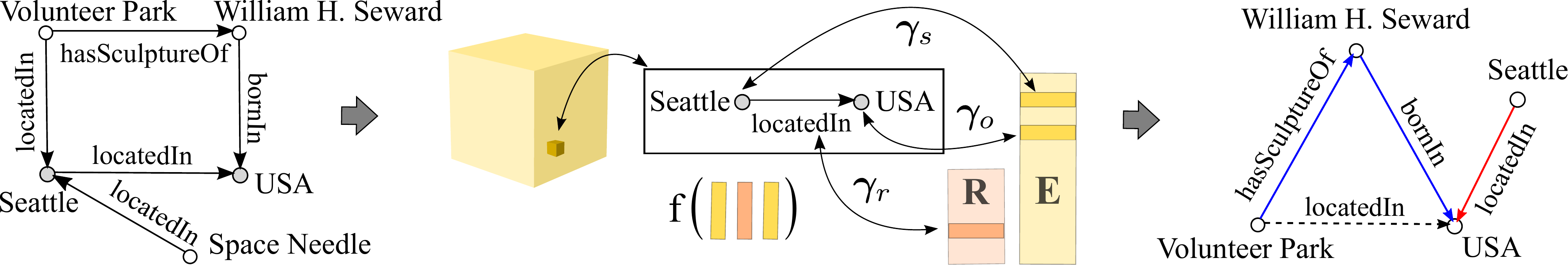}	
	\caption{\label{fig:overview} An illustration of the problem and the proposed method. A knowledge base of facts is represented as a sparse 3-dimensional matrix. Neural matrix factorization methods perform an implicit matrix decomposition by minimizing a loss function $f$ operating on the entity and relation representations. Gradient rollback tracks the parameter changes $\bfgamma$ caused by each sample during training. At test time, the aggregated updates are used to estimate triple influence without the need to retrain the model. Influence estimates are used to provide triples (blue and red) explaining the model's behavior for a test triple (dashed).}
\end{figure*}

\section{Neural Matrix Factorization}

We focus on neural matrix factorization models for link prediction in multi-relational graphs (also known as knowledge graph embedding models) for two reasons. First, these models are popular with a growing body of literature. There has been increasing interest in methods for explaining and debugging knowledge graph embedding methods~\cite{bianchi2020knowledge}. Second, matrix factorization methods for recommender systems can be seen as instances of neural matrix factorization where one uses pairs of entities instead of (entity, relation type, entity)-triples~\cite{guo2020survey}.

We consider the following general setting of representation learning for multi-relational graphs\footnote{Our   results with respect to $3$-dimensional matrices apply to $k$-dimensional matrices with $k \geq 2$.}.
There is a set of entities $\mathcal{E}$ and a set of relation types $\mathcal{R}$. 
A knowledge base $\mathcal{K}$ consists of a set of triples $d = (s, r, o)$ where $s \in \mathcal{E}$ is the subject (or head entity), $r \in \mathcal{R}$ the relation type, and $o \in \mathcal{E}$ the object (or tail entitiy) of the triple $d$.  
For each entity and relation type we \textit{learn} a vector representation with hidden dimension $h$.\footnote{Our results generalize to methods where relation types are represented with more than one vector such as in \textsc{RESCAL}~\cite{nickel2011three} or \textsc{ComplEx}~\cite{trouillon2016complex}.} For each entity $e$ and relation type $r$ we write $\mathbf{e}$ and $\mathbf{r}$, respectively,  for their vector representations. 
Hence, the set of parameters of a knowledge base embedding model consists of one matrix $\mathbf{E} \in \mathbb{R}^{|\mathcal{E}| \times h}$ and one matrix $\mathbf{R} \in \mathbb{R}^{|\mathcal{R}| \times h}.$ The rows of matrix $\mathbf{E}$ are the entity vector representations (embeddings) and the rows of $\mathbf{R}$ are the relation type vector representations (embeddings). To refer to the set of all parameters we often use the term $\bmw \in \Omega$ to improve readability, with $\Omega$ denoting the parameter space. Hence, $\bfe = w[e]$, that is, the embedding of entity $e$ is the part of $\bmw$ pertaining to entity $e$. Analogously for the relation type embeddings. Moreover, for every triple $d = (s, r, o)$, we write $\bmw[d] = (\bmw[s], \bmw[r], \bmw[o]) = (\bfs,\bfr,\bfo)$ to denote the triple of parameter vectors associated with $d$ for $\bmw \in \Omega$. We extend element-wise addition and multiplication to triples of vectors in the obvious way.

We can now define a \emph{scoring function} $\phi(\bmw; d)$ which, given the current set of parameters $\bmw$, maps each triple $d$ to a real-valued score. Note that, for a given triple $d$, the set of parameters involved in computing the value of $\phi(\bmw; d)$ are a small subset of $\bmw$. Usually, $\phi$ is a function of the vector representations of entities and relation types of the triple $d = (s, r, o)$, and we write $\phi(\bfs, \bfr, \bfo)$ to make this explicit. 
\begin{example}
	Given a triple $d = (s, r, o)$. The scoring function of \textsc{DistMult} is defined as 
	$\phi(\bmw; d) = \langle \bfs, \bfr, \bfo \rangle$, that is, the inner product of the three embeddings.
\end{example}
To train a neural matrix factorization model means running an iterative learning algorithm to find a set of paramater values $\bmw$ that minimize a loss function $\mathcal{L}(\bmw; \mathcal{D})$, which combines the scores of a set of training triples $\mathcal{D}$ into a loss-value.
We consider general iterative update rules such as stochastic gradient descent (SGD) of the form $G: \Omega \rightarrow \Omega$ which map a point
$\bmw$ in parameter space $\Omega$ to another point $G(\bmw) \in \Omega$. A typical loss function is $\mathcal{L}(\bmw; \mathcal{D}) = \sum_{d \in \mathcal{D}} \ell(\bmw; d)$ with
\begin{align}
\ell(\bmw; d = (s, r, o)) = -  \log \Pr(\bmw; o \mid s, r) \mbox { \ \ \textbf{and} \ \ } \label{loss-log} \\
\Pr(\bmw; o \mid s, r) = \frac{ \exp(\phi(\bmw; (s, r, o)))}{ \sum_{o'} \exp\left(\phi(\bmw; (s, r, o'))\right)}. \label{prob-norm} 
\end{align}
The set of $o'$ in the denominator is mostly a set of randomly  sampled entities (negative sampling) but recent results have shown that computing the softmax over all entities can be advantageous. Finally, once a model is trained, Equation \ref{prob-norm} can be used to determine for a test query $(s, r, ?)$ how likely each entity $o$ is by summing over all other entities $o'$.
\section{Gradient Rollback for Matrix Factorization}

A natural way of explaining the model's output for a test triple $d$ is to find the training triple $d'$ such that retraining the model without $d'$ changes (decreases \textbf{or} increases) $f(\bmw'; d)$ the most. The function $f$ can be the loss function. Most often, however, it is the scoring function as a proxy of the loss. 
\begin{equation}
\label{eq:refer_in_exp}
\begin{split}
d_{\mathtt{expl}} = \argmax_{d'\in \mathcal{D}} \left[ f(\bmw; d) - f(\bmw'; d)\right] \ \  \mbox{ \textbf{or} } & \\
d_{\mathtt{expl}}= \argmin_{d'\in \mathcal{D}} \left[ f(\bmw; d) - f(\bmw'; d) \right], &
\end{split}
\end{equation}

where $\bmw'$ are the parameters after retraining the model with $\mathcal{D} - \{d'\}$. This can directly be extended to obtaining the top-$k$ most influential triples. While this is an intuitive approach, solving the optimization problems above is too expensive in practice as it requires retraining the model $|\mathcal{D}|=n$ times for each triple one seeks to explain.\footnote{For example, \textsc{FB15k-237} contains 270$k$ training triples and training one model with TF2 on a RTX 2080 Ti GPU takes about 15 minutes. To explain one triple by retraining $|\mathcal{D}|=270k$ times would take over 2 months.} 

Instead of retraining the model, we propose gradient rollback (GR), an approach for tracking the influence each training triple has on a model's parameters during training. Before training and for each triple $d'$ and the parameters it can influence $\bmw(d')$, we initialize its influence values with zero: $\forall d', \bfgamma_{[d', \bmw(d')]} = \bm{0}$. We now record the changes in parameter values during training that each triple causes. With stochastic gradient descent (SGD), for instance, we iterate through the training triples $d' \in \mathcal{D}$ and record the changes it makes at time $t$, by setting its influence value to
\[\bfgamma_{[d', \bmw_{t+1}(d')]} \leftarrow \bfgamma_{[d', \bmw_{t}(d')]} - \alpha \nabla f(\bmw_t; d').\footnote{For any iterative optimizer such as SGD and Adam, the parameter updates are readily available and can be obtained efficiently during training.}\]
After training, $\gamma$ can be utilised as a look up table: for a triple $d$ that we want to explain, we look up the influence each triple $d'$ had on the weights $\bmw(d)$ of $d$ via $\bfgamma_{[d', \bmw(d)]} = (\bfgamma_s, \bfgamma_r, \bfgamma_o)$. Concretely, a triple $d$ has an influence on $d'$ at any point where the two triples have a common element wrt. their subject, relation or object. Consequently, explanations can be any triple $d'$ that has either an entity or the relation in common with $d$.

Let $\bmw'$ be the parameters of the model when trained on $\mathcal{D} - \{d'\}$. At prediction time, we now  approximate, for every test triple $d$, the difference $f(\bmw; d) - f(\bmw'; d)$  with $f(\bmw; d) - f(\bmw - \bfgamma_{[d', \bmw(d)]}; d)$, that is, by \emph{rolling back} the influence triple $d'$ had on the parameters of triple $d$ during training. We then simply choose the triple $d'$ that maximizes this difference. 

There are two crucial questions determining the utility of GR, which we address in the following:
\begin{enumerate}
	\item The resource overhead of GR during training and at test time; and
	\item How closely $f(\bmw - \bfgamma_{[d', \bmw(d)]}; d)$ approximates $f(\bmw'; d)$.
\end{enumerate}

\subsection{Resource Overhead of Gradient Rollback}

To address the first question, let us consider the computational and memory overhead for maintaining $\bfgamma_{[d, \bmw_t(d)]}$ during training for each $d$. The loss functions's normalization term is (in expectation) constant for all triples and can be ignored. We verified this empirically and it is indeed a reasonable assumption. Modern deep learning frameworks compute the gradients  and make them accessible in each update step. Updating $\bfgamma_{[d, \bmw_t(d)]}$ with the given gradients takes $O(h)$, that is, it is possible in constant time. Hence, computationally, the overhead is minimal. Now, to store $\bfgamma_{[d, \bmw_t(d)]}$ for every $d \in \mathcal{D}$ we need $h |\mathcal{D}| + 2 h |\mathcal{D}|$ floats. Hence, the memory overhead of gradient rollback is $3h |\mathcal{D}|$. This is about the same size as the parameters of the link prediction model itself.  In a nutshell:
\begin{mdframed}[backgroundcolor=white]
	\emph{Gradient rollback has a $O(h|\mathcal{D}|)$ computational and $3 h |\mathcal{D}|$ memory overhead during training.}
\end{mdframed}
At test time, we want to understand the computational complexity of computing 
\begin{equation}
\begin{split}
d_{\mathtt{expl}} = \argmax_{d'\in \mathcal{D}} f(\bmw; d)-f(\bmw - \bfgamma_{[d', \bmw(d)]}; d) \ \ \mbox { \textbf{or} } & \\
d_{\mathtt{expl}} = \argmin_{d'\in \mathcal{D}} f(\bmw; d)-f(\bmw - \bfgamma_{[d', \bmw(d)]}; d)  & 
\end{split}
\end{equation}
We have $\bfgamma_{[d', \bmw(d)]} \neq \bm{0}$ only if $d$ and $d'$ have at least one entity or relation type in common. Hence, to explain a triple $d$ we only have to consider triples $d'$ adjacent to $d$ in the knowledge graph, that is, triples where either of the arguments overlap. Let $\adj_{\max}$ and $\adj_{\avg}$ be the maximum and average number of adjacent triples in the knowledge graph. Then, we have the following:
\begin{mdframed}[backgroundcolor=white]
	\emph{Gradient rollback requires at most $\adj_{\max} + 1$ and on average $\adj_{\avg} + 1$  computations of the function $f$ to explain a test triple $d$.}
\end{mdframed}

\subsection{Approximation Error of Gradient Rollback}

To address the second question above, we need, for every pair of triples $d, d'$, to bound the expression
\[ \EX|f(\bmw - \bfgamma_{(d', \bmw(d))}; d) - f(\bmw'; d)|, \]
where $\bmw$ are the parameter values resulting from training $f$ on all triples $\mathcal{D}$ and $\bmw'$ are the parameter values resulting from training $f$ on $\mathcal{D} - \{d'\}$. If the above expression can be bounded and said bound is lower than what one would expect due to randomness of the iterative learning dynamics, the proposed gradient rollback approach would be highly useful. We use the standard notion of stability of learning algorithms~\cite{stability-hardt}. In a nutshell, our theoretical results will establish that:
\begin{mdframed}[backgroundcolor=white]
	\emph{Gradient rollback can approximate, for any $d' \in \mathcal{D}$, the changes of a scoring/loss function one would observe if a model were to be retrained on $\mathcal{D} - \{d'\}$. The approximation error is in expectation lower than known bounds on the stability of stochastic gradient descent.}
\end{mdframed}

\begin{definition}
	A function $f$ is L-Lipschitz if for all $u, v$ in the domain of $f$ we have
	$\norm{\nabla f(u)} \leq L$. This implies that $|f(u) - f(v)| \leq L\norm{u - v}$.
\end{definition}

We analyze the output of stochastic gradient descent on two data sets, $\mathcal{D}$ and $\mathcal{D} - \{d'\}$, that differ in precisely one triple.
If $f$ is $L$-Lipschitz for every example $d$, we have
\[ \EX |f(\bmw; d) - f(\bmw'; d)| \leq L \EX \norm{\bmw - \bmw'}\]
for all $\bmw$ and $\bmw'$. A vector norm in this paper is always the $2$-norm. 
Hence, we have $\EX|f(\bmw - \bfgamma_{(d', \bmw(d))}; d) - f(\bmw'; d)| \leq L\EX \norm{\bmw -\bfgamma_{(d', \bmw(d))} - \bmw'}$ and we can assess the approximation error by tracking the extent to which the parameter values $\bmw$ and $\bmw'$ from two coupled iterative learning dynamics diverge over time. 
Before we proceed, however, let us formally define some concepts and show that they apply to typical scoring functions of knowledge base embedding methods.

\begin{definition}
	A function $f: \Omega \rightarrow \mathbb{R}$ is $\beta$-smooth if for all $u, v$ in the domain of $f$ we have $\norm{\nabla f(u) - \nabla f(v)} \leq \beta \norm{u - v}.$
\end{definition}

To prove Lipschitz and $\beta$-smoothness properties of a function $f(\bmw; d)$ for all $\bmw \in \Omega$ and $d \in \mathcal{D}$, we henceforth assume that the norm of the entity and relation type embeddings is bounded by a constant $C > 0$. That is, we assume $\max_{r \in \mathcal{R}} \norm{\bmw[r]} \leq C$ and $\max_{e \in \mathcal{E}} \norm{\bmw[e]} \leq C$ for all $\bmw \in \Omega$. This is a reasonable assumption for two reasons. First, several regularization techniques constrain the norm of embedding vectors. For instance, the unit norm constraint, which was used in the original \textsc{DistMult} paper~\cite{distmult}, enforces that $C=1$. Second, even in the absence of such constraints we can assume a bound on the embedding vectors' norms as we are interested in the approximation error for and the stability of a given model that was obtained from running SGD a \emph{finite} number of steps using the \emph{same} parameter initializations. When running SGD, for $\mathcal{D}$ and all $\mathcal{D}-\{d'\}$, a finite number of steps with the same initialization, the  encountered parameters $\bmw$ and $\bmw'$ in each step of each run of SGD form a finite set. Hence, for our purposes, we can assume that  $f$'s domain $\Omega$ is compact. 
Given this assumption, we show that the inner product, which is used in several scoring functions, is $L$-Lipschitz and $\beta$-smooth on $\Omega$. The proofs of all lemmas and theorems can be found in the appendix.

\begin{replemma}{lemma-lipschitz-distmult}
	Let $\phi$ be the scoring function of \textsc{DistMult} defined as 
	$\phi(\bmw; d = (s, r, o)) = \langle \bfs, \bfr, \bfo \rangle$ with $\bmw(d) = (\bfs, \bfr, \bfo)$, and let $C$ be the bound on the norm of the embedding vectors for all $\bmw \in \Omega$.
	For a given triple $d = (s, r, o)$ and all $\bmw, \bmw' \in \Omega$, we have that
	\[|\phi(\bmw; d) - \phi(\bmw'; d)| \leq 2 C^2 \norm{\bmw  - \bmw'}.\]
\end{replemma}

\begin{replemma}{lemma-smooth-distmult}
	Let $\phi$ be the scoring function of \textsc{DistMult} defined as 
	$\phi(\bmw; d = (s, r, o)) = \langle \bfs, \bfr, \bfo \rangle$ with $\bmw(d) = (\bfs, \bfr, \bfo)$, and let $C$ be the bound on the norm of the embedding vectors for all $\bmw \in \Omega$.
	For a given triple $d = (s, r, o)$ and all $\bmw, \bmw' \in \Omega$, we have that
	\[ \norm{\nabla \phi(\bmw; d) - \nabla \phi(\bmw', d)} \leq 4C \norm{\bmw - \bmw'}.\]
\end{replemma}

Considering typical KG embedding loss functions and the softmax and sigmoid function being $1$-Lipschitz, this implies that the following theoretical analysis of \emph{gradient rollback} applies to a large class of neural matrix factorization models. 
Let us first define an additional property iterative learning methods can exhibit. 
\begin{definition}
	An update rule $G: \Omega \rightarrow \Omega$ is $\eta$-expansive if 
	\[ \sup_{\bfu, \bfv \in \Omega} \frac{\norm{G(\bfu) - G(\bfv)}}{\norm{\bfu-\bfv}} \leq \eta .\]
\end{definition}

Consider the gradient
updates $G_1, ..., G_T$ and $G'_1, ..., G'_T$ induced by running stochastic gradient descent on $\mathcal{D}$ and $\mathcal{D}-\{d'\}$, respectively. Every gradient update changes the parameters $\bmw$ and $\bmw'$ of the two coupled models.  Due to the difference in size, there is one gradient update $G_i$ whose corresponding update $G'_i$ does not change the parameters $\bmw'$. Again, note that there is always a finite set of parameters $\bmw$ and $\bmw'$ encountered during training. 

We can derive a stability bound for stochastic gradient descent run on the two sets, $\mathcal{D}$ and $\mathcal{D}-\{d'\}$. The following theorem and its proof are an adaptation of Theorem 3.21 in \citet{stability-hardt} and the corresponding proof. There are two changes compared to the original theorem. First, $f$ is not assumed to be a loss function but \emph{any} function that is $L$-Lipschitz and $\beta$-smooth. For instance, it could be that the loss function of the model is defined as $g(f(\cdot))$ for some Lipschitz and smooth function $g$. The proof of the original theorem does not rely on $f$ being a loss function and it only requires that $f(\cdot; d) \in [0, 1]$ is an $L$-Lipschitz and $\beta$-smooth function for all inputs. Second, the original proof assumed two training datasets of the same size that differ in one of the samples. We assume two sets, $\mathcal{D}$ and $\mathcal{D}-\{d'\}$, where $\mathcal{D}-\{d'\}$ has exactly one sample less than $\mathcal{D}$. The proof of the theorem is adapted to this setting. 

\begin{reptheorem}{theorem-stab-bound}
	Let $f(\cdot; d) \in [0, 1]$ be an $L$-Lipschitz and $\beta$-smooth function for
	every possible triple $d$ and let $c$ be the initial learning rate. Suppose we run SGD for $T$ steps with monotonically non-increasing step sizes $\alpha_t \leq c/t$ on two different sets of triples $\mathcal{D}$ and $\mathcal{D}-\{d'\}$. 
	Then, for any $d$, 
	\[ \EX |f(\bmw_T; d) - f(\bmw'_T; d)| \leq \frac{1+1/\beta c}{n-1}(cL^2)^{\frac{1}{\beta c + 1}}T^{\frac{\beta c}{\beta c + 1}},\]
	with $\bmw_T$ and $\bmw_T'$ the parameters of the two models after running SGD. We name the right term in the above inequality $\Lambda_{\mbox{stab-nc}}$.
\end{reptheorem}
The assumption $f(\cdot; d) \in [0, 1]$ of the theorem is fulfilled if we use the loss from Equation~\ref{loss-log}, as long as we are interested in the stability of the probability distribution of Equation~\ref{prob-norm}. That is because the cross-entropy loss applied to a softmax distribution is $1$-Lipschitz and $1$-smooth where the derivative is taken with respect to the logits. Hence, the $L$-Lipschitz and $\beta$-smoothness assumption holds also for the probability distribution of Equation~\ref{prob-norm}. In practice, estimating the influence on the probability distribution from Equation~\ref{prob-norm} is what one is interested in and what we evaluate in our experiments. 

The following lemma generalizes the known $(1 + \alpha \beta)$-expansiveness property of the gradient update rule~\cite{stability-hardt}. It establishes that the increase of the distance between $\bmw - \bfgamma$ and $\bmw'$ after one step of gradient descent is at most as much as the increase in distance between two parameter vectors corresponding to $\mathcal{D}$ and $\mathcal{D}-\{d'\}$ after one step of gradient descent. 

\begin{replemma}{lemma-expanse}
	Let $f: \Omega \rightarrow \mathbb{R}$ be a function and let $G(\bmw) = w - \alpha \nabla f(\bmw)$ be the gradient update rule with step size $\alpha$. Moreover, assume that $f$ is $\beta$-smooth. Then, for every $\bmw, \bmw', \bfgamma \in \Omega$ we have
	\[ \norm{G(\bmw) - \bfgamma - G(\bmw')} \leq \norm{\bmw - \bfgamma - \bmw'} + \alpha \beta \norm{\bmw - \bmw'}. \]
\end{replemma}

\begin{replemma}{lemma-norm-bound}
Let $f(\cdot; d)$ be $L$-Lipschitz and $\beta$-smooth function. Suppose we run SGD for $T$ steps on two sets of triples $\mathcal{D}$ and $\mathcal{D}-\{d'\}$ for any $d' \in\mathcal{D}$ and with learning rate $\alpha_t$ at time step $t$. Moreover, let $\Delta_t = \EX\left[\norm{\bmw_t - \bmw_t'} \mid \norm{\bmw_{t_0} - \bmw_{t_0}'}=0\right]$ and $\hat{\Delta}_t = \EX\left[\norm{\bmw_t - \bfgamma_{[d', \bmw_t(d)]} - \bmw_t'} \mid \norm{\bmw_{t_0} - \bmw_{t_0}'}=0\right]$ for some $t_0 \in \{1, ..., n\}$ . Then, for all $t \geq t_0$, 
\[ \hat{\Delta}_{t+1} < \left(1 - \frac{1}{n}\right)(1 + \alpha_t \beta)\Delta_t + \frac{1}{n}\left(\Delta_t + \alpha_t L\right). \]
\end{replemma}
The lemma has broader implications since it can be extended to other update rules $G$ as long as they fulfill an $\eta$-expansive property and their individual updates are bounded. For each of these update rules the above lemma holds. 

The following theorem establishes that the approximation error of gradient rollback is smaller than a known stability bound of SGD. It uses Lemma~\ref{lemma-norm-bound} and the proof of Theorem~\ref{theorem-stab-bound}.

\begin{reptheorem}{theorem-bound-gr}
	Let $f(\cdot; d) \in [0, 1]$ be an $L$-Lipschitz and $\beta$-smooth function. Suppose we run SGD for $T$ steps with monotonically non-increasing step sizes $\alpha_t \leq c/t$ on two sets of triples $\mathcal{D}$ and $\mathcal{D}-\{d'\}$. Let  $\bmw_T$ and $\bmw_T'$, respectively, be the resulting parameters.
	Then, for any triple $d$ that has at least one element in common with $d'$ we have, 
	\[ \EX|f(\bmw_T - \bfgamma_{[d', \bmw_T(d)]}; d) - f(\bmw'_T; d)| < \Lambda_{\mbox{stab-nc}}.\]
\end{reptheorem}

The previous results establish a connection between estimating the influence of training triples on the model's behavior using GR and the stability of SGD when used to train the model. An interesting implication is that regularization approaches that improve the stability (by reducing the Lipschitz constant and/or the expansiveness properties of the learning dynamics, cf.~\cite{stability-hardt}) also reduce the error bound of GR. We can indeed verify this empirically.
\section{Related Work}

The first neural link prediction method for multi-relational graphs performing an implicit matrix factorization is RESCAL~\cite{nickel2011three}. Numerous scoring functions have since been proposed. Popular examples are \textsc{TransE}~\cite{bordes2013translating}, \textsc{DistMult}~\cite{distmult}, and \textsc{ComplEx}~\cite{trouillon2016complex}. Knowledge graph embedding methods have been mainly evaluated through their 
accuracy on link prediction tasks. A number of papers has recently shown that with appropriate hyperparameter tuning, \textsc{ComplEx}, \textsc{DistMult}, and \textsc{RESCAL} are highly competitive scoring functions, often achieving state-of-the-art results~\cite{kadlec2017knowledge,RuffinelliBG20,jain2020knowledge}.  There are a number of proposals for combining rule-based and matrix factorization methods~\cite{rocktaschel2015injecting,guo2016jointly,minervini2017adversarial}, which can make link predictions  more interpretable. In contrast, we aim to generate faithful explanations for non-symbolic knowledge graph embedding methods.  

There has been an increasing interest in understanding model behavior through adversarial attacks~\cite{biggio2013security, papernot2016limitations,dong2017towards,ebrahimi-etal-2018-hotflip}. Most of these approaches are aimed at visual data. There are, however, several approaches that consider adversarial attacks on graphs~\cite{dai2018adversarial,zugner2018adversarial}. While analyzing attacks can improve model interpretability, the authors focused on neural networks for single-relational graphs and the task of node classification. 
For a comprehensive discussion of adversarial attacks on graphs we refer the reader to a recent survey \cite{chen2020survey}. There is prior work on adversarial samples for KGs but with the aim to improve accuracy and not model interpretability~\cite{minervini2017adversarial,kbgan2018}. 

There are two recent papers that directly address the problem of explaining graph-based ML methods. First, \textsc{GNNExplainer}~\cite{ying2019gnnexplainer} is a method for explaining the predictions of graph neural networks and, specifically, graph convolutional networks for node classification. Second, the work most related to ours proposes \textsc{Criage} which aims at estimating the influence of triples in KG embedding methods~\cite{criage:naacl19}. Given a triple, their method only considers a neighborhood to be the set of triples with the same object.  Moreover, they derive a first-order approximation of the influence in line with work on influence functions \cite{koh2017-influence-functions}. In contrast, GR tracks the changes made to the parameters during training and uses the aggregated contributions to estimate influence. In addition, we establish a theoretical connection to the stability of learning systems. 
Influence functions, a concept from robust statistics, were applied to black-box
models for assessing the changes in the loss caused by changes in the training data \cite{koh2017-influence-functions}. The paper also proposed several strategies to make influence functions more efficient. It was shown in prior work \cite{criage:naacl19}, however, that influence functions are not usable for typical knowledge base embedding methods as they scale very poorly. Consequently our paper is the first to offer an efficient and theoretically founded method of tracking influence in matrix factorization models for explaining prediction via providing the most influential training instances.

\section{Experiments}

\begin{table*}[t!]
	\centering
	\begin{center}
		\begin{tabular}{|lrrr|rrr|rrr||lrrr|rrr|rrr|}
			\hline
			&\multicolumn{9}{c}{PD\%}&&\multicolumn{9}{c|}{TC\%}\\
			&\multicolumn{3}{c}{\textsc{Nations}}&\multicolumn{3}{c}{\textsc{FB15k-237}}&\multicolumn{3}{c}{\textsc{Movielens}}&&\multicolumn{3}{c}{\textsc{Nations}}&\multicolumn{3}{c}{\textsc{FB15k-237}}&\multicolumn{3}{c|}{\textsc{Movielens}}\\
			&1&10&\textsc{all}&1&10&\textsc{all}&1&10&\textsc{all}&&1&10&\textsc{all}&1&10&\textsc{all}&1&10&\textsc{all}\\
			\hline
			NH&54&66&82&59&67&83&53&52&92&NH&18&36&70&35&45&59&$\hphantom{0}$3&14&72\\
			GR&93&97&100&77&82&96&68&82&100&GR&38&83&97&38&58&85&20&38&100\\
			$\Delta\uparrow$&39&31&18&18&15&13&15&30&8&$\Delta\uparrow$&20&47&27&3&13&26&17&24&28\\
			\hline
			\hline
			NH&59&68&80&72&83&91&53&61&91&NH&13&47&76&69&70&79&5&13&77\\
			GR&90&95&100&80&88&99&73&71&100&GR&38&85&99&65&81&91&29&48&100\\
			$\Delta\uparrow$&31&27&20&8&5&8&20&10&9&$\Delta\uparrow$&25&38&23&-4&11&12&24&35&23\\
			\hline
		\end{tabular}
		\caption{Results (\textsc{DistMult} at the top, \textsc{ComplEx} at the bottom) of removing a set of training triples (of size $1$, $10$, or \textsc{all}), randomly chosen from triples adjacent to the test triples (NH) or by using gradient rollback (GR). For \textsc{all} we delete on average ($\pm$ standard deviation), \textsc{Nations}: 261$\pm56$, \textsc{FB15k-237}: 2.9k$\pm2.3k$ and \textsc{Movielens}: 16.7k$\pm5k$ (\textsc{DistMult}). The average number of adjacent triples ($\pm$ standard deviation) is, \textsc{Nations}: 508$\pm101$, \textsc{FB15k-237}: 5.5k$\pm4.5k$ and \textsc{Movielens}: 23.7k$\pm7.8k$. GR removes sets that lead to a larger change in probability and top-1 predictions (difference to NH is given in row $\Delta\uparrow$).}
		\label{tab:results}
	\end{center}
\end{table*}

\textbf{Identifying explanations.} We analyze the extent to which GR can approximate the true influence of a triple (or set of triples). For a given test triple $d=(s,r,o)$ and a trained model with parameters $\bmw$, we use GR to identify a set of training triples $\mathcal{S} \subseteq \mathcal{D}$ that has the highest influence on $d$. To this end, we first identify the set of triples $\mathcal{N}$ adjacent to $d$ (that is, triples that contain at least one of $s$, $r$ or $o$) and compute $\Delta(d', d) = \Pr(\bmw; o \mid s, r) - \Pr(\bmw-\bfgamma_{[d', \bmw(d)]}; o \mid s, r)$ for each $d' \in \mathcal{N}$. We then let $\mathcal{S}$ be the set resulting from picking (a) exactly $k$ triples $d' \in \mathcal{N}$ with the $k$ largest values for $\Delta(d', d)$ or (b) \emph{all} triples with a positive $\Delta(d', d)$. We refer to the former as GR-$k$ and the latter as GR-\textsc{all}.

To evaluate if the set of chosen triples $\mathcal{S}$ are faithful explanations \cite{JacoviGoldberg:20} for the prediction, we follow the evaluation paradigm ``RemOve And Retrain (ROAR)'' of \citet{roar}:
We let $\mathcal{D}' = \mathcal{D}-\mathcal{S}$ and retrain the model from scratch with training set $\mathcal{D}'$ leading to a new model with parameters $\bmw'$. After retraining, we can now observe $\Pr(\bmw'; o \mid s, r)$, which is the true probability for $d$ when removing the explanation set $\mathcal{S}$ from the training set, and we can use this to evaluate GR.\footnote{We fix all random seeds and use the same set of negative samples during (re-)training to avoid additional randomization effects.}
Since it is expensive to retrain a model for each test triple, we restrict the analysis to explaining only the triple $(s, r, \hat{o})$ with $\hat{o} = \argmax_o \Pr(\bmw; o \mid s, r)$ for each test query $(s, r, ?)$, that is, the triple with the highest score according to the model.

\textbf{Evaluation metrics.} We use two different metrics to evaluate GR. First, if the set $\mathcal{S}$ contains triples influential to the test triple $d$, then the probability of $d$ under the new model (trained without $\mathcal{S}$) should be smaller, that is, $\Pr(\bmw'; o \mid s, r) < \Pr(\bmw; o \mid s, r)$. We measure the ability of GR to identify triples causing a \textit{P}robability \textit{D}rop and name this measure PD\%. A PD of 100\% would imply that each set $\mathcal{S}$ created with GR always caused a drop in probability for $d$ after retraining with $\mathcal{D}'$. In contrast, when removing random training triples, we would expect a PD\% close to 50\%.
An additional way to evaluate GR is to measure whether the removal of $\mathcal{S}$ causes the newly trained model to predict a different top-$1$ triple, that is, $\argmax_o \Pr(\bmw; o \mid s, r) \neq \argmax_o \Pr(\bmw'; o \mid s, r)$. This measures the ability of GR to select triples causing the \textit{T}op-1 prediction to \textit{C}hange and we name this TC\%. If the removal of $\mathcal{S}$ causes a top-1 change, it suggests that the training samples most influential to the prediction of triple $d$ have been removed. This also explores the ability of GR to identify triples for effective removal attacks on the model. We compare GR to two baselines: NH-$k$ removes exactly $k$ random triples adjacent to $d$ and NH-\textsc{all} randomly removes the same number of triples as GR-\textsc{all} adjacent to $d$. In an additional experiment, we also directly compare GR with \textsc{Criage}~\cite{criage:naacl19}.

\textbf{Datasets \& Training.} We use \textsc{DistMult}~\cite{distmult} and \textsc{ComplEx}~\cite{trouillon2016complex} as scoring functions since they are popular and competitive~\cite{kadlec2017knowledge}. We report results on three datasets: two knowledge base completion (\textsc{Nations}~\cite{nations}, \textsc{FB15k-237} \cite{fb15k-237}) and one recommendation dataset (\textsc{Movielens} \cite{movielens}). \textsc{Movielens} contains triples of 5-star ratings users have given to movies; as in prior work the set of entities is the union of movies and users and the set of relations are the ratings~\cite{pezeshkpour-etal-2018-embedding}. When predicting movies for a user, we simply filter out other users. Statistics and hyperparameter settings are in the appendix. Since retraining is costly, we only explain the top-1 prediction for a set of 100 random test triples for both \textsc{FB15k-237} and \textsc{Movielens}. For \textsc{Nations} we use the entire test set.

We want to emphasize that we \emph{always retrain completely from scratch} and use hyperparameter values typical for state-of-the-art KG completion models, leading to standard results for \textsc{DistMult} (see Table \ref{tab:eval} in the Appendix). Prior work either retrained from pretrained models~\cite{koh2017-influence-functions,criage:naacl19} or used non-standard strategies and hyperparameters to ensure model convergence~\cite{criage:naacl19}. 

\begin{table}[t!]
	\centering
	\begin{center}
		\begin{tabular}{|lrrrr|rrrr|}
			\hline
			&\multicolumn{4}{c}{PD\%}&\multicolumn{4}{c|}{TC\%}\\
			&1&3&5&10&1&3&5&10\\
			\hline
			NH&49&50&51&60&3&13&14&20\\
			\textsc{Criage}&91&93&94&95&16&36&48&68\\
			GR-O&\textbf{92}&94&\textbf{97}&97&25&\underline{48}&\underline{62}&68\\
			GR&\textbf{92}&\textbf{96}&\textbf{97}&\textbf{98}&\textbf{\underline{27}}&\textbf{\underline{51}}&\textbf{\underline{66}}&\textbf{73}\\
			\hline
		\end{tabular}
		\caption{Results on \textsc{Nations}, using \textsc{DistMult} with a \textsc{sigmoid} activation function and for $k=\{1, 3, 5, 10\}$. Bold marks the best results; underlined results mark a statistical significance with regards to \textsc{Criage} at $p\leq0.01$ using an approximate randomization test; all results are statistically signifiant with regards to NH. \textsc{Criage} performs worse than both GR and GR-O, especially with regards to TC. Furthermore, \textsc{Criage} considers only training triples with the same object as an explanation and is significantly slower.}
		\label{tab:results_criage}
	\end{center}
\end{table}

\begin{figure*}[t!]
\centering
\begin{subfigure}{.32\textwidth}
	\centering
	\includegraphics[width=0.8\linewidth]{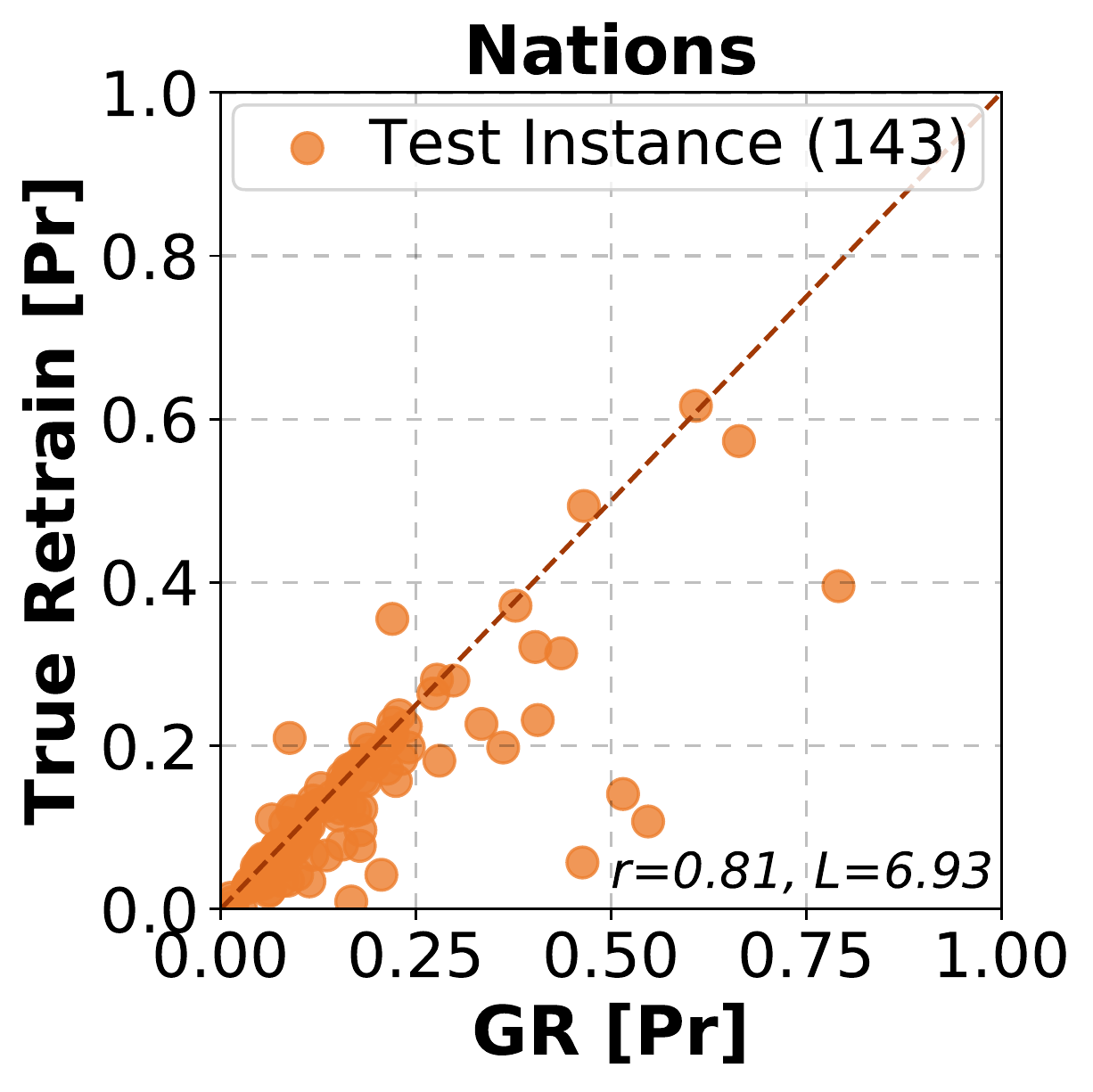}
\end{subfigure}
\begin{subfigure}{.32\textwidth}
	\centering
	\includegraphics[width=0.8\linewidth]{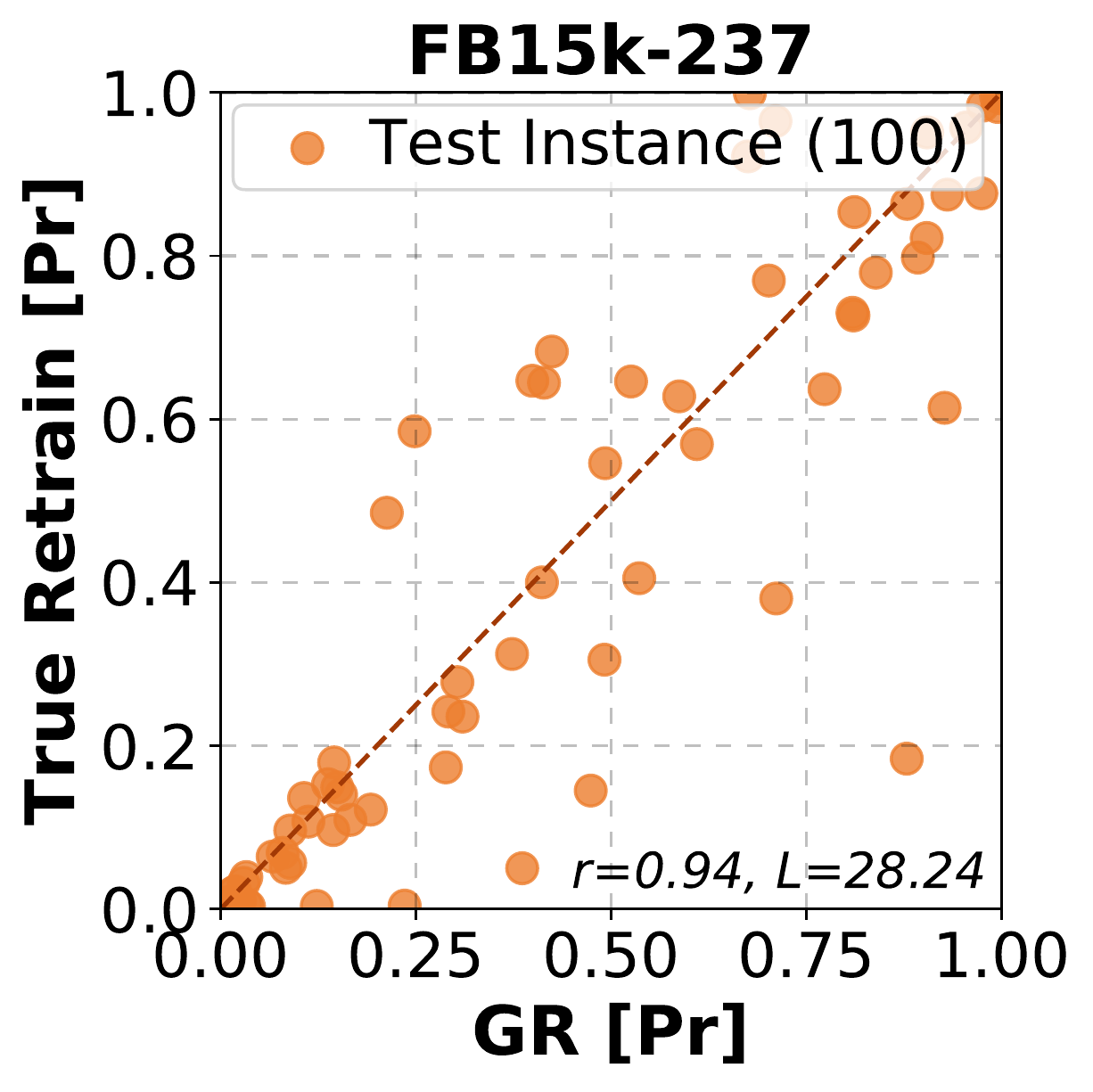}
\end{subfigure}
\begin{subfigure}{.32\textwidth}
	\centering
	\includegraphics[width=0.8\linewidth]{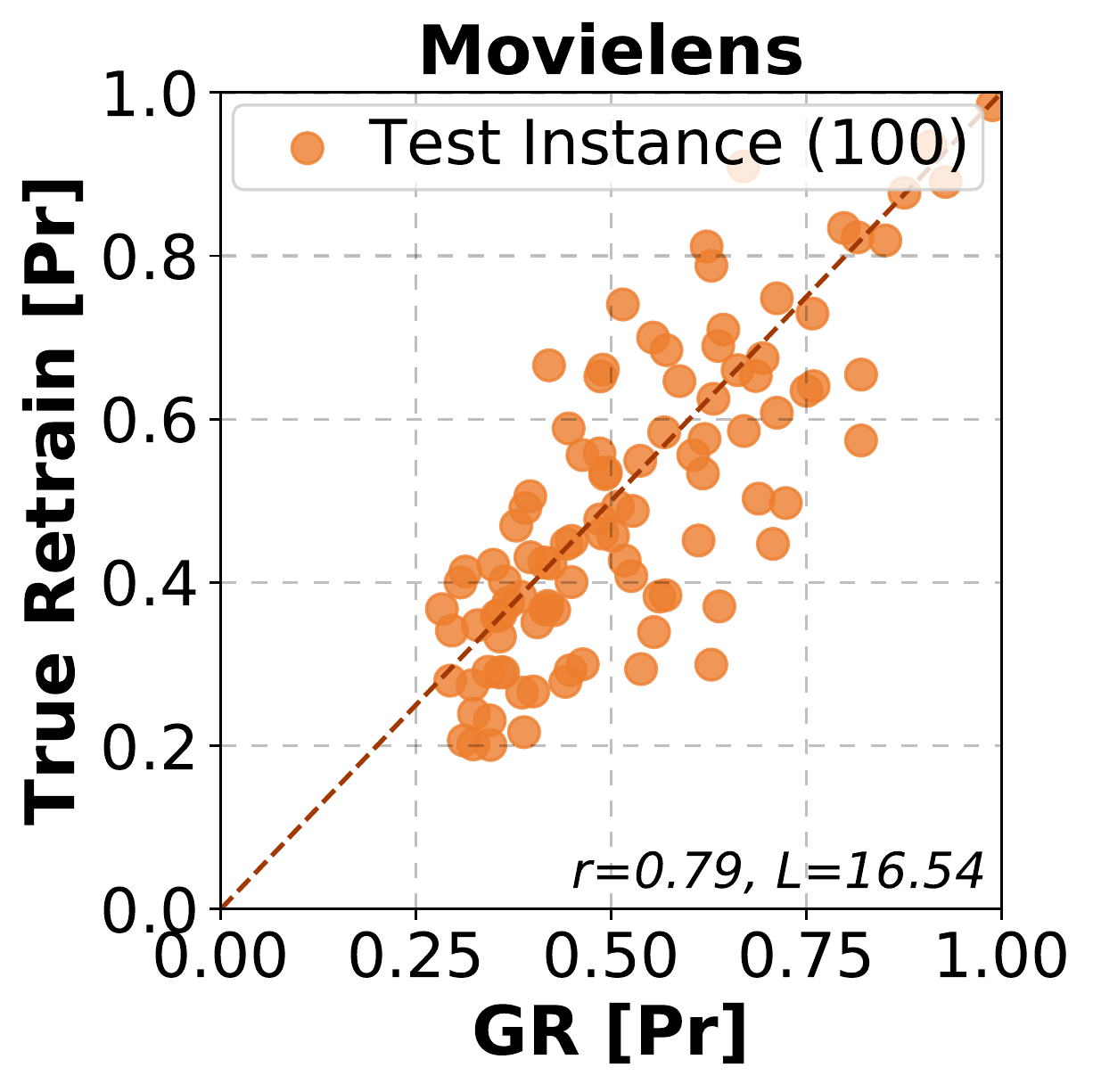}
\end{subfigure}
\begin{subfigure}{.32\textwidth}
	\centering
	\includegraphics[width=0.8\linewidth]{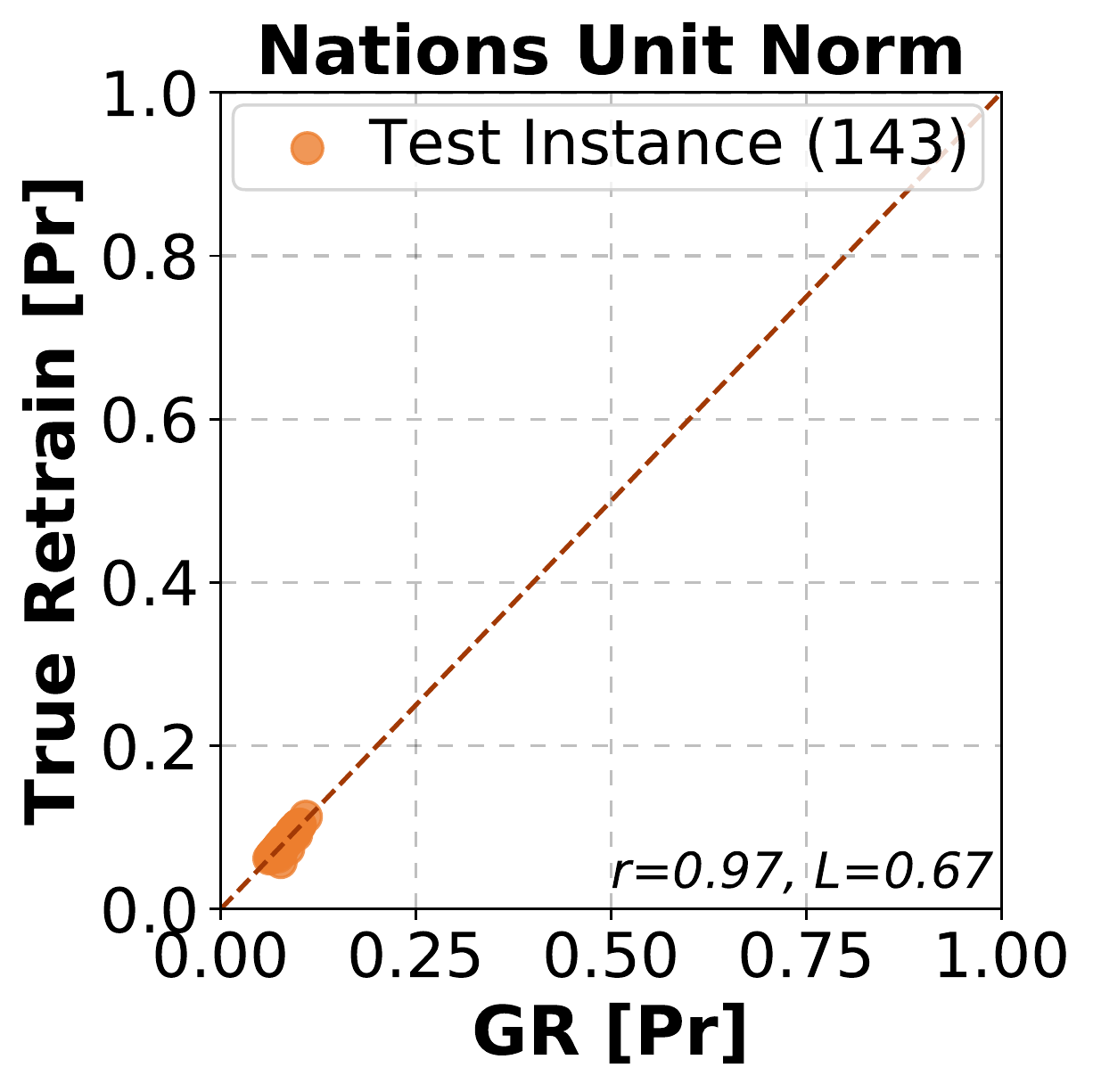}
\end{subfigure}
\begin{subfigure}{.32\textwidth}
	\centering
	\includegraphics[width=0.8\linewidth]{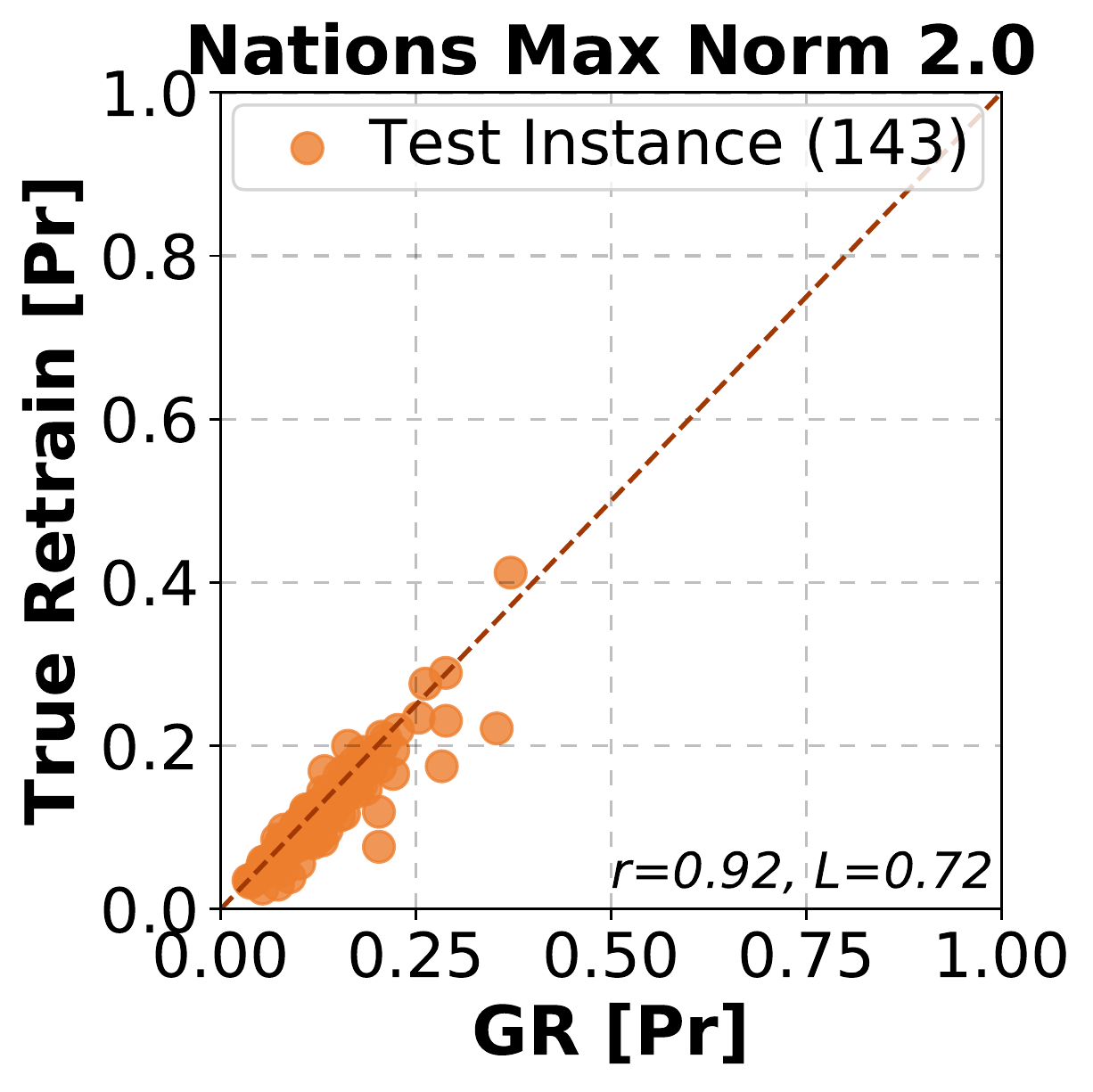}
\end{subfigure}
\begin{subfigure}{.32\textwidth}
	\centering
	\includegraphics[width=0.8\linewidth]{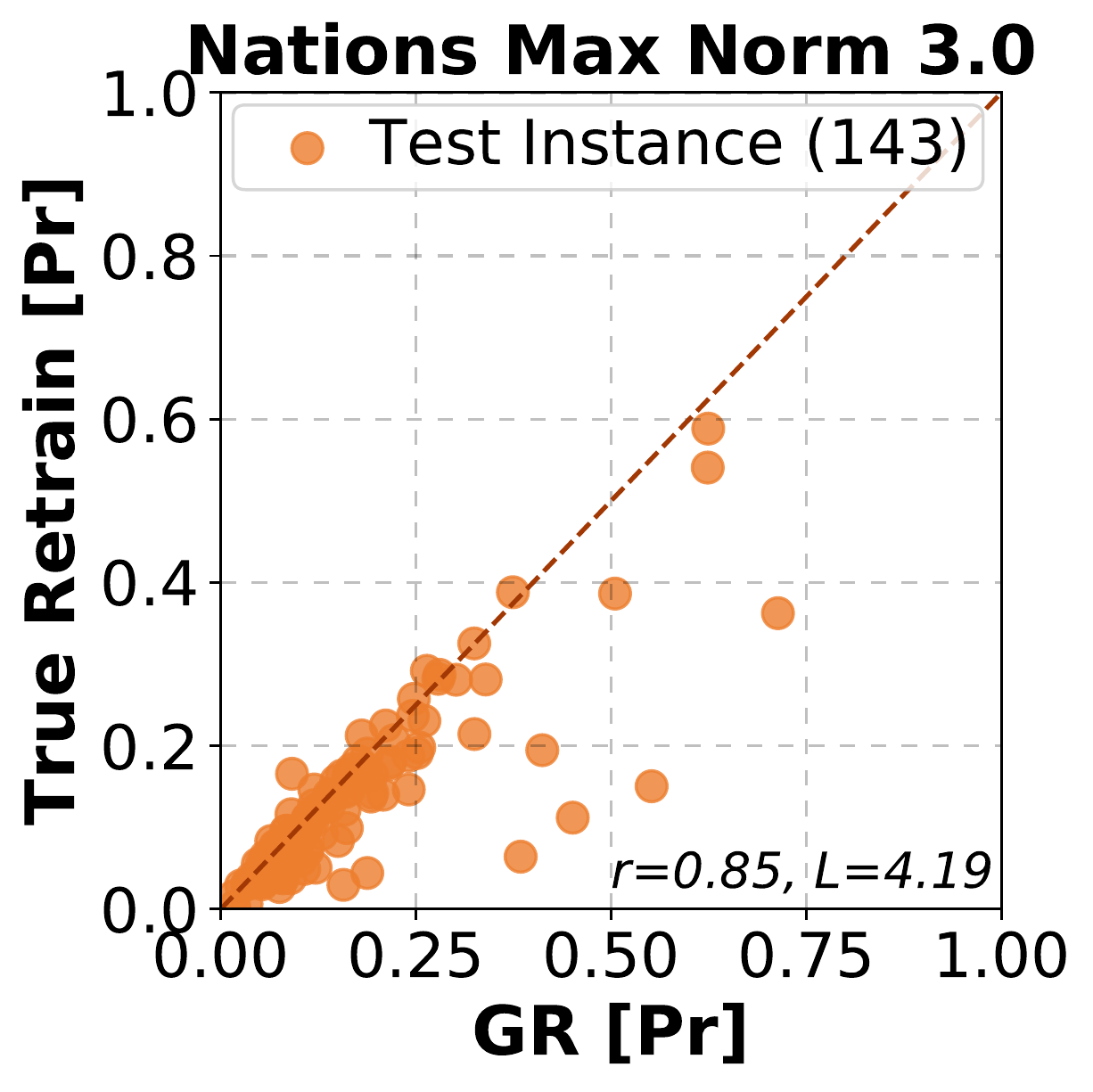}
\end{subfigure}
\caption{\label{fig:scatterplot_GR_vs_retrain}  Scatter plots illustrating the correlation between the probability values estimated by GR and after a retraining of the model, as well as Pearson correlation value \textit{r} and Lipschitz constant \textit{L}. The high correlations indicate that GR is able to accurately approximate the probability and, therefore, the change in probability, of a triple after retraining the model without the triple that GR deemed to cause the highest change in probability. The bottom row shows results when imposing constraints on the weights for the \textsc{Nations} dataset: the stronger the constraint (unit norm $>$ maximum norm $2.0 > 3.0 >$ None), the smaller the Lipschitz constant and the better the Pearson correlation.}
\end{figure*}


\textbf{Results.} The top half of Table~\ref{tab:results} lists the results using \textsc{DistMult} for GR and NH. For \textsc{Nations} and \textsc{DistMult}, removing 1 triple from the set of adjacent triples (NH-1) is close to random, causing a drop in probability (PD\%) about half of the time. In contrast, removing the 1 training instance considered most influential by GR (GR-1), gives a PD of over 90\%. GR-1  leads to a top-1 change (TC) in about 40\% of the cases. This suggests that there is more than one influential triple per test triple. When removing all triples identified by GR (GR-\textsc{all}) we observe PD and TC values of nearly 100\%. In contrast, removing the same number of adjacent triples randomly leads to a significantly lower PC and TC. In fact, removing the highest ranked triple (GR-1) impacts the model behavior more than deleting about half the adjacent triples at random (NH-\textsc{all}) in terms of PD\%.

For \textsc{FB15k-237} we observe that GR is again better at identifying influential triples compared to the NH baselines. 
Moreover, only when the NH method deletes over half of the adjacent triples at random (NH-\textsc{all}) does it perform on par with GR-10 in terms of PD\% and TC\%. Crucially, deleting one instance at random (NH-$1$) causes a high TC\% compared to the other two datasets. This suggests that the model is less stable for this dataset. As our theoretical results show, GR works better on more stable models, and this could explain why GR-\textsc{all} is further away from reaching a TC\% value of 100 than on the other two datasets.

For \textsc{Movielens} GR is again able to outperform all respective NH baselines. Furthermore, on this dataset GR-\textsc{all} is able to achieve perfect PD\% and TC\% values. Additionally, GR shows a substantial increase for PD\% and TC\% when moving from $k=1$ to $k=10$, suggesting that more than one training triple is required for an explanation. At the same time, NH-$10$ still does not perform better than random in terms of PD\%. The bottom half of Table~\ref{tab:results} reports the results using \textsc{ComplEx}, which show a similar patterns as \textsc{DistMult}. In the appendix, Figure~\ref{fig:boxplot} depicts the differences between GR and NH: GR is always better at identifying highly influential triples. 

\textbf{Approximation Quality.} In this experiment we select, for each test triple, the training triple $d'$ GR deemed most influential. We then retrain without $d'$ and compare the predicted probability of the test triple with the probability (PR) after retraining the model. Figure~\ref{fig:scatterplot_GR_vs_retrain} shows that GR is able to accurately estimate the probability of the test triple after retraining. For all datasets the correlation is rather high. We also confirm empirically that imposing a constraint on the weights (enforcing unit norm or a maximum norm), reduces the Lipschitz constant of the model and in turn reduces the error of GR, as evidenced by the increase in correlation. In addition to the quantitative results, we also provide some typical example explanations generated by GR and their interpretation in the appendix. Figure~\ref{fig:example-graphs} also shows two qualitative examples.

\textbf{Comparison to \textsc{Criage}.} Like GR, \textsc{Criage} can be used to identify training samples that have a large influence on a prediction. However, \textsc{Criage} has three disadvantages: (1) it only considers training instances with the same object as the prediction as possible explanations; (2) it only works with a \textsc{sigmoid} activation function and not for \textsc{softmax}; (3) retrieving explanations is time consuming, e.g. on \textsc{Movielens} GR can identify an explanation in 3 minutes whereas \textsc{Criage} takes 3 hours. Consequently, we only run \textsc{Criage} on the smaller \textsc{Nations} dataset. For a direct comparison, we introduce GR-O, which like \textsc{Criage} only considers training instances as possible explanations if they have the same object as the prediction. Results for the top-$k \in \{1, 3, 5, 10\}$ are reported in Table \ref{tab:results_criage}. Both GR and GR-O outperform \textsc{Criage} and GR performs best overall for all $k$ and both metrics.

\begin{figure}
	\centering
	\includegraphics[width=0.5\linewidth]{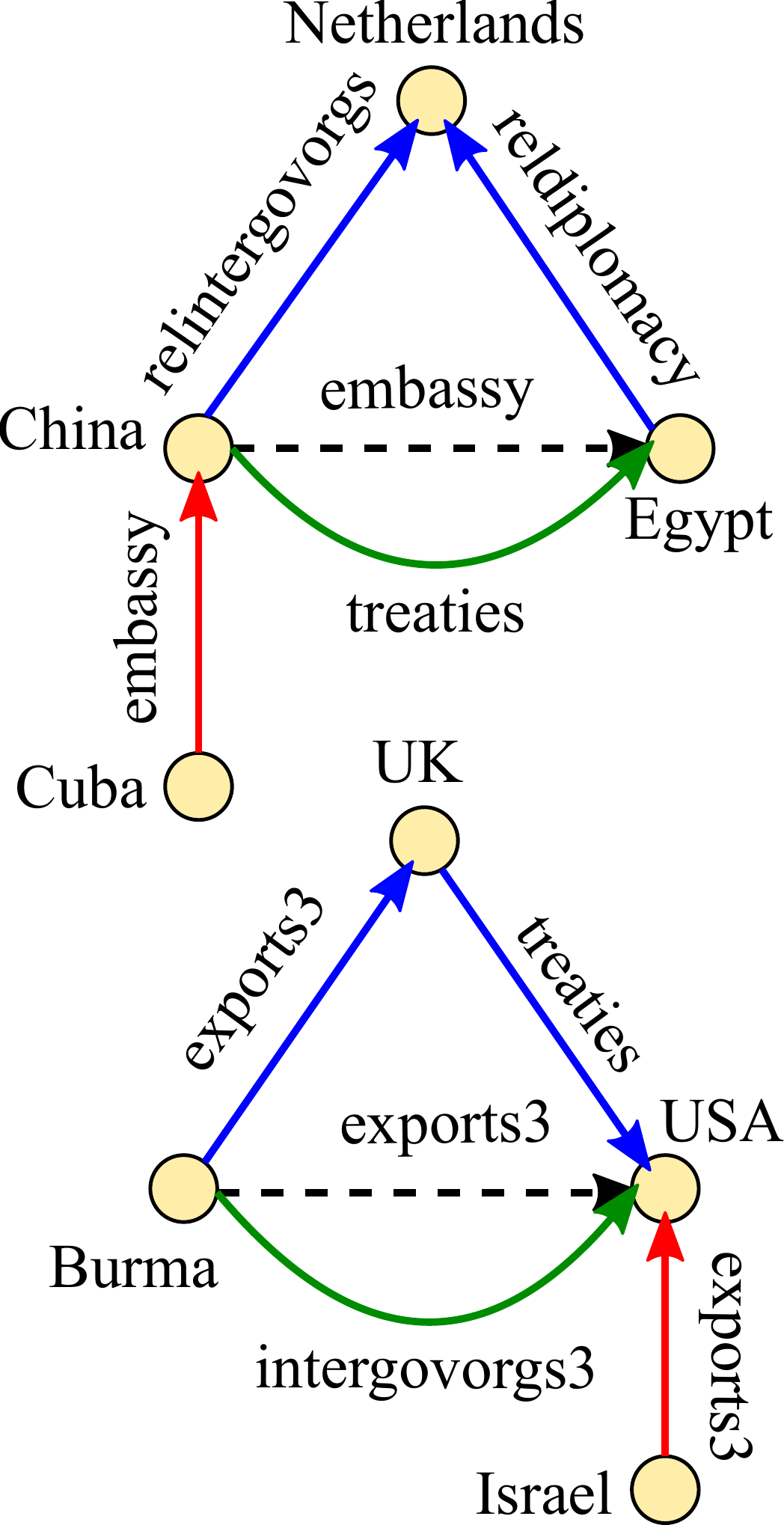}
	\caption{\label{fig:example-graphs}  Some example explanations generated using GR. The dashed line indicates the test triple. Depicted are triples deemed highly influential on model behavior by GR (i) with the same head and tail as test triple (green), (ii) with the same relation type (red), and (iii) a pair of triples via a third entity (blue).}
\end{figure}
\section{Conclusion}
Gradient rollback (GR) is a simple yet effective method to track the influence of training samples on the parameters of a model. Due to its efficiency it is  suitable for large neural networks, where each training instance touches only a moderate number of parameters. Instances of this class of models are neural matrix factorization models. To showcase the utility of GR, we first established theoretically that the resource overhead is minimal. Second, we showed that the difference of GR's influence approximation and the true influence on the model behaviour is smaller than known bounds on the stability of stochastic gradient descent. This establishes a link between influence estimation and the stability of models and shows that, if a model is stable, GR's influence estimation error is small. We showed empirically that GR can successfully identify sets of training samples that cause a drop in performance if a model is retrained without this set. In the future we plan to apply GR to other models.

\section{Acknowledgements}
We thank Takuma Ebisu for his time and insightful feedback.

\section{Broader Impact}
This paper addresses the problem of explaining and analyzing the behavior of machine learning models. More specifically, we propose a method that is tailored to a class of ML models called matrix factorization methods. These have a wide range of applications and, therefore, also the potential to provide biased and otherwise inappropriate predictions in numerous use cases. For instance, the predictions of a recommender system might be overly gender or race specific. We hope that our proposed influence estimation approach can make these models more interpretable and can lead to improvements of production systems with respect to the aforementioned problems of bias, fairness, and transparency. At the same time, it might also provide a method for an adversary to find weaknesses of the machine learning system and to influence its predictions making them more biased and less fair. In the end, we present a method that can be used in several different applications. We believe, however, that the proposed method is not inherently problematic as it is agnostic to use cases and does not introduce or discuss problematic applications. 

\bibliography{gr-bib}

\clearpage

\begin{appendices}

\section{Definitions, Lemmas, and Proofs}

The following lemma simplifies the derivation of Lipschitz and smoothness bounds for tuple-based scoring functions. 

\begin{lemma}
	\label{lemma-lipschitz1-3}
	Let $f: \mathbb{R}^k \times \mathbb{R}^l \times \mathbb{R}^m \rightarrow \mathbb{R}^n: (x, y, z) \rightarrow f(x, y, z)$ be Lipschitz continuous relative to, respectively, $x$, $y$, and $z$. Then $f$ is Lipschitz continuous as a function $\mathbb{R}^{k+l+m} \rightarrow \mathbb{R}^n$. More specifically, 
	$| f(u_x, u_y, u_z) - f(v_x, v_y, v_z)| \leq 2 \max\{L_x, L_y, L_z\} \norm{(u_x,u_y,u_z) - (v_x,v_y,u_z)}.$
\end{lemma}

\begin{proof}
	$|f(u_x, u_y, u_z) - f(v_x, v_y, v_z)|$
	\begin{linenomath*}
		\begin{align*}
		= & \lvert\lvert f(u_x, u_y, u_z) - f(u_x, u_y, v_z)\\& + f(u_x, u_y, v_z) - f(v_x, v_y, v_z) \rvert\rvert \\
		\leq &   \norm{f(u_x, u_y, u_z) - f(u_x, u_y, v_z)} \\&+ \norm{f(u_x, u_y, v_z) - f(v_x, v_y, v_z)} \\
		= &   \norm{f(u_x, u_y, u_z) - f(u_x, u_y, v_z)} \\&+ \lvert\lvert f(u_x, u_y, v_z) - f(u_x, v_y, v_z) \\&+ f(u_x, v_y, v_z) - f(v_x, v_y, v_z) \rvert\rvert\\		
		\leq & \norm{f(u_x, u_y, u_z) - f(u_x, u_y, v_z)} \\&+ \norm{f(u_x, u_y, v_z) - f(u_x, v_y, v_z)} \\&+ \norm{f(u_x, v_y, v_z) - f(v_x, v_y, v_z)}\\	
		\leq & L_z \norm{u_z - v_z} + L_y \norm{u_y - v_y} + L_x \norm{u_x - v_x} \\
		\leq & \max\{L_x, L_y, L_z\} \left( \norm{u_x - v_x} + \norm{u_y - v_y} \right.\\&\left. + \norm{u_z - v_z} \right) \\
		= & \max\{L_x, L_y, L_z\} \left( \norm{(u_x,\bfnull) - (v_x,\bfnull)} \right.\\&\left.+ \norm{(\bfnull, u_y) - (\bfnull,v_y)} + \norm{u_z - v_z} \right) \\
		\leq & \max\{L_x, L_y, L_z\} \left( \sqrt{2} \norm{(u_x,u_y) - (v_x,v_y)} \right.\\&\left. + \norm{u_z - v_z} \right) \\
		= & \max\{L_x, L_y, L_z\} \left( \sqrt{2} \norm{(u_x,u_y,\bfnull) - (v_x,v_y,\bfnull)}  \right.\\&\left.+ \norm{(\bfnull,\bfnull,u_z) - (\bfnull,\bfnull,v_z)} \right) \\
		\leq & \max\{L_x, L_y, L_z\} \left( 2 \norm{(u_x,u_y,u_z) - (v_x,v_y,u_z)} \right) \\
		= & 2 \max\{L_x, L_y, L_z\} \norm{(u_x,u_y,u_z) - (v_x,v_y,u_z)}.
		\end{align*}
	\end{linenomath*}
\end{proof}

We can now use Lemma~\ref{lemma-lipschitz1-3} to derive the Lipschitz constant for the \textsc{DistMult} scoring function and a bound on its smoothness. 

\begin{lemma}
\label{lemma-lipschitz-distmult}
	Let $\phi$ be the scoring function of \textsc{DistMult} defined as
	$\phi(\bmw; d = (s, r, o)) = \langle \bfs, \bfr, \bfo \rangle$ with $\bmw(d) = (\bfs, \bfr, \bfo)$, and let $C$ be the bound on the norm of the embedding vectors for all $\bmw \in \Omega$.
	For a given triple $d = (s, r, o)$ and all $\bmw, \bmw' \in \Omega$, we have that
	\[|\phi(\bmw; d) - \phi(\bmw'; d)| \leq 2 C^2 \norm{\bmw  - \bmw'}.\]
\end{lemma}

\begin{proof}
Let $\bmw, \bmw' \in \Omega$ and let $(\bfs, \bfr, \bfo)= (\bmw[s], \bmw[r], \bmw[o])$ and $(\bfs',\bfr',\bfo') = (\bmw'[s],\bmw'[r],\bmw'[o])$. Without loss of generality, assume that $\norm{\bfs \bfr} \geq \norm{\bfs' \bfr'}$. 
	We first show that 
	\begin{linenomath*}
		\begin{align*}
		| \langle \bfs, \bfr, \bfo \rangle - \langle \bfs, \bfr, \bfo'\rangle| & = |\langle \bfs\bfr, (\bfo-\bfo')\rangle| \\
		& \leq \norm{\bfs \bfr}\norm{\bfo-\bfo'}  \\
		& = \norm{\bfs \bfr}  \norm{(\bfs,\bfr,\bfo)-(\bfs, \bfr, \bfo')} \\
		& \leq \norm{\bfs}\norm{\bfr}  \norm{(\bfs,\bfr,\bfo)-(\bfs, \bfr, \bfo')} \\
		& \leq C^2 \norm{(\bfs,\bfr,\bfo)-(\bfs, \bfr, \bfo')}.
		\end{align*}
	\end{linenomath*}
The first inequality is by Cauchy-Schwarz. The last inequality is by the assumption that the norm of the embedding vectors for all $\bmw \in \Omega$ is bounded by $C$. We can repeat the above derivation for, respectively, $\bfr'$ and $\bfs'$. Hence, we can apply Lemma~\ref{lemma-lipschitz1-3} to conclude the proof.
\end{proof}

\begin{lemma}
\label{lemma-smooth-distmult}
	Let $\phi$ be the scoring function of \textsc{DistMult} defined as 
	$\phi(\bmw; d = (s, r, o)) = \langle \bfs, \bfr, \bfo \rangle$ with $\bmw(d) = (\bfs, \bfr, \bfo)$, and let $C$ be the bound on the norm of the embedding vectors for all $\bmw \in \Omega$.
	For a given triple $d = (s, r, o)$ and all $\bmw, \bmw' \in \Omega$, we have that
	\[ \norm{\nabla \phi(\bmw; d) - \nabla \phi(\bmw', d)} \leq 4C \norm{\bmw - \bmw'}.\]
\end{lemma}

\begin{proof}
Let $\bmw, \bmw' \in \Omega$ and let $(\bfs, \bfr, \bfo)= (\bmw[s], \bmw[r], \bmw[o])$ and $(\bfs',\bfr',\bfo') = (\bmw'[s],\bmw'[r],\bmw'[o])$. 
	We first show that 		
		\begin{align*}		
		&\hspace{-1.5cm} \norm{\nabla \langle \bfs, \bfr, \bfo \rangle - \nabla \langle \bfs, \bfr, \bfo'\rangle} \\
		& = \norm{(\bfr \bfo, \bfs\bfo, \bfs\bfr) - (\bfr \bfo', \bfs\bfo', \bfs\bfr)} \\
		& = \norm{ \bfr(\bfo-\bfo'), \bfs(\bfo-\bfo'), \bm{0})} \\
		& \leq \norm{ \bfr(\bfo-\bfo') } + \norm{\bfs(\bfo-\bfo')} \\
		& \leq \norm{\bfr}\norm{\bfo-\bfo'} + \norm{\bfs}\norm{\bfo-\bfo'} \\
		& = \left(\norm{\bfr} + \norm{\bfs}\right)\norm{\bfo-\bfo'} \\
		& = (\norm{\bfr} + \norm{\bfs})  \norm{(\bfs,\bfr,\bfo)-(\bfs, \bfr, \bfo')} \\
		& \leq 2C \norm{(\bfs,\bfr,\bfo)-(\bfs, \bfr, \bfo')}.
		\end{align*}
	Hence, we have that $\nabla\langle \bfs, \bfr, \bfo\rangle$ is $2C$-Lipschitz with respect to $\bfo$. We can repeat the above derivation for, respectively, $\bfr'$ and $\bfs'$. Hence, we can apply Lemma~\ref{lemma-lipschitz1-3} to conclude the proof.
\end{proof}

We first need another definition.

\begin{definition}
	An update rule $G: \Omega \rightarrow \Omega$ is $\sigma$-bounded if 
	\[ \sup_{\bmw \in \Omega} \norm{\bmw - G(\bmw)} \leq \sigma .\]
\end{definition}

Stochastic gradient descent (SGD) is the algorithm resulting from performing stochastic gradient
updates $T$ times where the indices of training examples are randomly chosen. There are two popular sampling approaches for gradient descent type algorithms. One is to choose an example uniformly at random at each step. The other is to choose a random permutation and cycle through the examples repeatedly. We focus on the latter sampling strategy but the presented results can also be extended to the former.

We can now prove the stability of SGD on two sets of training samples.

\begin{theorem}
	\label{theorem-stab-bound}
	Let $f(\cdot; d)\in[0,1]$ be an $L$-Lipschitz and $\beta$-smooth function for
	every possible triple $d$ and let $c$ be the initial learning rate. Suppose we run SGD for $T$ steps with monotonically non-increasing step sizes $\alpha_t \leq c/t$ on two different sets of triples $\mathcal{D}$ and $\mathcal{D}-\{d'\}$. 
	Then, for any $d$, 
	\begin{linenomath*}
		\begin{align*} \EX |f(\bmw_T; d) - f(\bmw'_T; d)| \leq \Lambda_{\mbox{stab-nc}} \\
		:=  \frac{1+1/\beta c}{n-1}(cL^2)^{\frac{1}{\beta c + 1}}T^{\frac{\beta c}{\beta c + 1}},\end{align*}
	\end{linenomath*}
	with $\bmw_T$ and $\bmw_T'$ the parameters of the two models after running SGD.
\end{theorem}

\begin{proof}
The proof of the theorem follows the notation and strategy of the proof of the uniform stability bound of SGD in \cite{stability-hardt}. We provide some of the details here as we will use a similar proof idea for the approximation bound for gradient rollback. 
	By Lemma~3.11 in \cite{stability-hardt}, we have for every $t_0 \in \{1, ..., n\}$
	\[ \EX|f(\bmw_T; d) - f(\bmw'_T; d)| \leq \frac{t_0}{n} + L\EX[\delta_T \mid \delta_{t_0} = 0] \]
	where $\delta_t = \norm{\bmw_t - \bmw_t'}$. Let $\Delta_{t} = \EX[\delta_t \mid \delta_{t_0}=0]$. We will state a bound $\Delta_t$ as a function of $t_0$ and then follow exactly the proof strategy of \cite{stability-hardt}.
	
	At step $t$ with probability $1-1/n$ the triple selected by SGD exists in both sets $\mathcal{D}$ and $\mathcal{D}-\{d'\}$. In this case we can use the  $(1 + \alpha \beta)$-expansivity of the update rule~\cite{stability-hardt} which follows from our smoothness assumption. With probability $1/n$ the selected triple is $d'$. Since $d'$ is the triple missing in $\mathcal{D}-\{d'\}$ the weights $\bmw'$ are not updated at all. This is justified because we run SGD on both sets a set number of epochs, each with $|\mathcal{D}|$ steps and, therefore, do not change $\bmw'$ with probability $1/n$ because the training set corresponding to $\bmw'$ has one sample less. Hence, in this case we use the $\alpha L$-bounded property of the update as a consequence of the $\alpha L$-boundedness of the gradient descent update rule.
	We can now apply Lemma 2.5 from \cite{stability-hardt} and linearity of expectation to conclude that for every $t \geq t_0,$
	\[\Delta_{t+1} \leq \left(1 - \frac{1}{n}\right) (1 + \alpha_t\beta)\Delta_t + \frac{1}{n}\Delta_t + \frac{\alpha_tL}{n} .\]
	The theorem now follows from unwinding the recurrence relation and optimizing for $t_0$ in the exact same way as done in prior work~\cite{stability-hardt}.
\end{proof}

\begin{lemma}
	\label{lemma-expanse}
	Let $f: \Omega \rightarrow \mathbb{R}$ be a function and let $G(\bmw) = w - \alpha \nabla f(\bmw)$ be the gradient update rule with step size $\alpha$. Moreover, assume that $f$ is $\beta$-smooth. Then, for every $\bmw, \bmw', \bfgamma \in \Omega$ we have
	\[ \norm{G(\bmw) - \bfgamma - G(\bmw')} \leq \norm{\bmw - \bfgamma - \bmw'} + \alpha \beta \norm{\bmw - \bmw'}. \]
\end{lemma}

\begin{proof}
	\begin{linenomath*}
		\begin{align*}
		&\norm{G(\bmw) - \bfgamma - G(\bmw')} \\
		= &\norm{\bmw - \alpha \nabla f(\bmw) - \bfgamma - (\bmw' - \alpha \nabla f(\bmw')} \\
		\leq&  \norm{\bmw - \bfgamma - \bmw'} + \alpha \norm{\nabla f(\bmw) - \nabla f(\bmw')}  \\
		\leq&  \norm{\bmw - \bfgamma - \bmw'} + \alpha \beta \norm{\bmw - \bmw'}.
		\end{align*}
	\end{linenomath*}
	The first equality follows from the definition of the gradient update rule. The first inequality follows from the triangle inequality. The second inequality follows from the definition of $\beta$-smoothness.
\end{proof}

\begin{lemma}
\label{lemma-norm-bound}
Let $f(\cdot; d)$ be $L$-Lipschitz and $\beta$-smooth function. Suppose we run SGD for $T$ steps on two sets of triples $\mathcal{D}$ and $\mathcal{D}-\{d'\}$ for any $d' \in\mathcal{D}$ and with learning rate $\alpha_t$ at time step $t$. Moreover, let $\Delta_t = \EX\left[\norm{\bmw_t - \bmw_t'} \mid \norm{\bmw_{t_0} - \bmw_{t_0}'}=0\right]$ and $\hat{\Delta}_t = \EX\left[\norm{\bmw_t - \bfgamma_{[d', \bmw_t(d)]} - \bmw_t'} \mid \norm{\bmw_{t_0} - \bmw_{t_0}'}=0\right]$ for some $t_0 \in \{1, ..., n\}$ . Then, for all $t \geq t_0$, 
\[ \hat{\Delta}_{t+1} < \left(1 - \frac{1}{n}\right)(1 + \alpha_t \beta)\Delta_t + \frac{1}{n}\left(\Delta_t + \alpha_t L\right). \]
\end{lemma}

\begin{proof}
We know from the proof of Theorem~\ref{theorem-stab-bound} that the following recurrence holds
\begin{equation}
\label{eq-lemma-norm-bound}
\Delta_{t+1} \leq \left(1 - \frac{1}{n}\right)(1 + \alpha_t \beta)\Delta_t + \frac{1}{n}\left(\Delta_t + \alpha_t L\right).
\end{equation}
We first show that when using gradient rollback (GR), we have that 
\begin{equation}
\label{new-recurrence}
\hat{\Delta}_{t+1} \leq \left(1 - \frac{1}{n}\right)(\hat{\Delta}_{t} + \alpha_t \beta\Delta_{t}) + \frac{1}{n}\hat{\Delta}_t.
\end{equation}
At step $t$ of SGD, with probability $1-\frac{1}{n}$, the triple selected is in both $\mathcal{D}$ and $\mathcal{D}-\{d'\}$. In this case, by Lemma~\ref{lemma-expanse}, we have that $\hat{\Delta}_{t+1} \leq \hat{\Delta}_{t} + \alpha_t \beta\Delta_{t}$. With probability $\frac{1}{n}$ the selected example is $d'$,  in which case we have, due to the behavior of GR, that
{\small
\begin{align*}
&\hat{\Delta}_{t+1} \\
&\;\;= \EX\left[\norm{\bmw_{t+1} - \bfgamma_{[d', \bmw_{t+1}(d)]} - \bmw_{t+1}'} \mid \norm{\bmw_{t_0} - \bmw_{t_0}'}=0\right] \\
&\;\; =  \EX\left[\norm{\bmw_{t+1} - \bfgamma_{[d', \bmw_{t+1}(d)]} - \bmw_{t}'} \mid \norm{\bmw_{t_0} - \bmw_{t_0}'}=0\right]\\
&\;\; =  \EX\left[\big|\!\big|\bmw_{t} - \alpha_t \nabla f(\bmw_t; d') - \left( \bfgamma_{[d', \bmw_{t}(d)]} \right.\right.\\
&\;\; \;\;\;\;\;\;\; \left. - \alpha_t \nabla f(\bmw_t; d')\big)- \bmw_{t}'\big|\!\big| \mid \norm{\bmw_{t_0} - \bmw_{t_0}'}=0\right] \\
&\;\; =  \EX\left[\norm{\bmw_{t} -  \bfgamma_{[d', \bmw_{t}(d)]} - \bmw_{t}'} \mid \norm{\bmw_{t_0} - \bmw_{t_0}'}=0\right]  \\
&\;\; = \hat{\Delta}_{t},
\end{align*}}

This proves Equation~\ref{new-recurrence}.

Let us now define the recurrences $A_{t+1}$ and $B_{t+1}$ for the upper bounds of $\Delta_t$ and $\hat{\Delta}_t$, respectively: 

\begin{equation}
A_{t+1} = \left(1 - \frac{1}{n}\right)(1 + \alpha_t \beta)A_t + \frac{1}{n}\left(A_t + \alpha_t L\right),
\end{equation}

\begin{equation}
B_{t+1} = \left(1 - \frac{1}{n}\right)(B_t + \alpha_t \beta A_t) + \frac{1}{n}B_t.
\end{equation}

Now, let 
\begin{align*}
C_{t+1} & = A_{t+1} - B_{t+1} \\
 & = \left(1 - \frac{1}{n}\right)(1 + \alpha_t \beta)A_t + \frac{1}{n}\left(A_t + \alpha_t L\right) \\
 &\;\;\;\;- \left(1 - \frac{1}{n}\right)(B_t + \alpha_t \beta A_t) - \frac{1}{n}B_t \\
 & = \left(1 - \frac{1}{n}\right) A_t + \frac{1}{n}\left( A_t + \alpha_t L\right) - B_t \\
& = \left( A_t - B_t \right) + \frac{1}{n} \alpha_t L\\
& = C_t + \frac{1}{n} \alpha_t L.
\end{align*}

Let us now derive the initial conditions of the recurrence relations. By the assumption $\norm{\bmw_{t_0} - \bmw_{t_0}'}=0$, we have $\Delta_{t_0} = 0$, $\hat{\Delta}_{t_0} = 0$, $A_{t_0} = 0$, $B_{t_0} = 0$, and $C_{t_0} = 0$. Now, since $\alpha_t L > 0$ we have for all $t \geq t_0$ that $C_{t+1} = A_{t+1} - B_{t+1} > 0$ and, therefore, $B_{t+1} <  A_{t+1}$. 
It follows that for all $t \geq t_0$ we have 
\begin{align*}
&\hat{\Delta}_{t+1} \leq B_{t+1} < A_{t+1} \\
&= \left(1 - \frac{1}{n}\right)(1 + \alpha_t \beta)\Delta_t + \frac{1}{n}\left(\Delta_t + \alpha_t L\right).\end{align*} 
This concludes the proof. 
\end{proof}

The following theorem establishes that the approximation error of gradient rollback is smaller than the stability bound of SGD. It uses Lemma~\ref{lemma-norm-bound} and the proof of Theorem~\ref{theorem-stab-bound}.

\begin{theorem}
\label{theorem-bound-gr}
	Let $f(\cdot; d)\in[0,1]$ be an $L$-Lipschitz and $\beta$-smooth loss function. Suppose we run SGD for $T$ steps with monotonically non-increasing step sizes $\alpha_t \leq c/t$ on two sets of triples $\mathcal{D}$ and $\mathcal{D}-\{d'\}$. Let  $\bmw_T$ and $\bmw_T'$, respectively, be the resulting parameters.
	Then, for any triple $d$ that has at least one element in common with $d'$ we have, 
	\[ \EX|f(\bmw_T - \bfgamma_{[d', \bmw_T(d)]}; d) - f(\bmw'_T; d)| < \Lambda_{\mbox{stab-nc}}.\]
\end{theorem}

\begin{proof}
The proof strategy follows that of Theorem~\ref{theorem-stab-bound} by analyzing how the parameter vectors from two different runs of SGD, on two different training sets, diverge. 

By Lemma~3.11 in \cite{stability-hardt}, we have for every $t_0 \in \{1, ..., n\}$ 
\begin{align}
\label{result-from-hardt}
&\EX|f(\bmw_T; d) - f(\bmw'_T; d)| \leq \frac{t_0}{n} + \\ \notag
&L\EX\left[\norm{\bmw_t - \bmw_t'} \mid \norm{\bmw_{t_0} - \bmw_{t_0}'}=0\right].
\end{align}
Let 
\[\Delta_{t} = \EX\left[\norm{\bmw_t - \bmw_t'} \mid \norm{\bmw_{t_0} - \bmw_{t_0}'}=0\right]\]
and
\[\hat{\Delta}_t = \EX\left[\norm{\bmw_t - \bfgamma_{[d', \bmw_t(d)]} - \bmw_t'} \mid \norm{\bmw_{t_0} - \bmw_{t_0}'}=0\right].\]
We will state bounds $\Delta_t$ and $\hat{\Delta}_r$ as a function of $t_0$ and then follow the proof strategy of \cite{stability-hardt}.

From Lemma~\ref{lemma-norm-bound} we know that 
\[ \hat{\Delta}_{t+1} < \left(1 - \frac{1}{n}\right)(1 + \alpha_t \beta)\Delta_t + \frac{1}{n}\left(\Delta_t + \alpha_t L\right). \]
For the recurrence relation 
\[ \Delta_{t+1} \leq \left(1 - \frac{1}{n}\right)(1 + \alpha_t \beta)\Delta_t + \frac{1}{n}\left(\Delta_t + \alpha_t L\right) \]
we know from the proof of Theorem 3.12 in \cite{stability-hardt} that 
\[ \Delta_T \leq \frac{L}{\beta(n-1)}\left(\frac{T}{t_0}\right)^{\beta c}.\]

Hence, we have by Lemma~\ref{lemma-norm-bound} that, for some $\epsilon > 0$,
\[\hat{\Delta}_{T} \leq \frac{L}{\beta(n-1)}\left(\frac{T}{t_0}\right)^{\beta c} - \epsilon.\]

Plugging this into Equation~\ref{result-from-hardt}, we get
\begin{align*} &\EX|f(\bmw_T - \bfgamma_{[d', \bmw_T(d)]}; d) - f(\bmw'_T; d)| \\
&\leq \frac{t_0}{n} + \frac{L^2}{\beta(n-1)}\left(\frac{T}{t_0}\right)^{\beta c} - L \epsilon. \end{align*}

	The theorem now follows from optimizing for $t_0$ in the exact same way as done in prior work~\cite{stability-hardt}.
\end{proof}

\section{Experimental Details \& Qualitative Examples}
\label{app:statistics}

\paragraph{Experimental Details.} Dataset statistics and the number of negative samples and dimensions can be found in Table \ref{tab:datasets}. The implementation uses Tensorflow 2. 
Furthermore, the optimizer used for all models is Adam \cite{adam} with an initial learning rate (LR) of 0.001 and a decaying LR wrt the number of total steps. To obtain exact values for GR, the batch size is set to 1. Because the evaluation requires the retraining of models, we set the epoch size for \textsc{FB15k-237} and \textsc{Movielens} to 1. For \textsc{Nations}, we set it to 10 and use the model from the epoch with the best MRR on the validation set, which was epoch 10. For \textsc{Nations} and \textsc{FB15k-237}, we use the official splits. \textsc{Movielens}\footnote{\url{https://grouplens.org/datasets/MovieLens/100k/}; 27th May 2020} is the 100$k$ split, where the file ua.base is the training set and the first 5$k$ of ua.test are the validation set and the remaining are the test set. We train \textsc{Nations} and \textsc{FB15k-237} with \textit{softmax\_cross\_entropy\_with\_logits}. For \textsc{Movielens} we employ \textit{sigmoid\_cross\_entropy\_with\_logits} in order to test if GR can also be used under a different loss and because this loss is standardly used in recommender systems. MRR, Hits@1 and Hits@10 of the models can be found in Table \ref{tab:eval}. 
Experiments were run on CPUs for \textsc{Nations} and on GPUs (GeForce GTX 1080 Ti) for the other two datasets on a Debian GNU/Linux 9.11 machine with 126GB of RAM. Computing times for the different steps and datasets can be found in Table \ref{tab:time}. Due to the expensive evaluation step, each experimental setup was run once. All seeds were fixed using number 42.

\begin{table*}[h]
	\begin{center}
		\begin{tabular}{lllllllll}
			\toprule
			&Train&Dev&Test&\#E&\#R&\#Neg&\#Dim\\
			\midrule
			\textsc{Nations}&1592&199&201&14&55&13&10\\
			\textsc{FB15k-237}&272$k$&17.5$k$&20.5$k$&14$k$&237&500&100\\
			\textsc{Movielens}&90$k$&5$k$&4.4$k$&949&1.7$k$&500&200\\
			\bottomrule
		\end{tabular}
		\caption{Dataset statistics, including number of entities (\#E) and number of relations (\#R), as well as hyperparameter settings, number of negative samples (\#Neg) and number of dimensions (\#Dim).}
		\label{tab:datasets}
	\end{center}
\end{table*}

\begin{table}[h]
	\begin{center}
		\begin{tabular}{llll}
			\toprule
			&MRR&Hits@1&Hits@10\\
			\midrule
			\textsc{Nations}&58.88&38.31&97.51\\
			\textsc{FB15k-237}&25.84&19.16&40.18\\
			\textsc{Movielens}&61.81&37.60&n/a\\
			\hline
			\hline
			\textsc{Nations}&60.41&39.80&97.01\\
			\textsc{FB15k-237}&25.11&18.15&40.05\\
			\textsc{Movielens}&61.76&37.51&n/a\\
			\bottomrule
		\end{tabular}
		\caption{MRR, Hits@1 and Hits@10 for the main models of the three datasets, \textsc{DistMult} is at the top and \textsc{ComplEx} at the bottom. (\textsc{Movielens} has 5 ratings that can be predicted, hence Hits@10 is trivially 100.)}
		\label{tab:eval}
	\end{center}
\end{table}

\begin{table}[h]
	\begin{center}
		\begin{tabular}{lcc}
			\toprule
			&Step 1&Step 2\\
			\midrule
			\textsc{Nations}&1&0.07$\pm$0.02\\
			\textsc{FB15k-237}&17&6$\pm$7\\
			\textsc{Movielens}&7&12$\pm$4\\
			\bottomrule
		\end{tabular}
		\caption{Computing times in minutes for the different steps and datasets. Step 1: Train a main model. Step 2: Generate explanation for one triple (the time for this step depends on the number of adjacent training triples, here averaged over 5 random triples). Step 3: Evaluating GR/NH takes the time of Step 1 times the test set size.}
		\label{tab:time}
	\end{center}
\end{table}

\paragraph{Qualitative Examples.}
Beside the question whether GR is able to identify the most influential instances is actually if those would be considered as an explanation by a human. For that, we performed a qualitative analyzes by inspecting the explanation picked by \textsc{GR-1}. We focused on \textsc{FB15k-237} as the triples are human-readable and because they describe general domains like geography or people. Below we present three examples which reflect several other identified cases.\\
\textit{to-be-explained}: “Lycoming County” - “has currency” - “United States Dollar” \\
\textit{explanation}: “Delaware County” - “has currency” - “United States Dollar”

The model had to predicted the currency of Lycoming County and GR picked as the most likely explanation that Delaware County has the US-Dollar as currency. Indeed, both are located in Pennsylvania; thus, it seems natural that they also have the same currency. \\
\textit{to-be-explained}: “Bandai” – “has\_industry” – “Video game” \\
\textit{explanation}: “Midway Games” – “has\_industry” – “Video game” 

Namco (Bandai belongs to Namco) initially distributed its games in Japan, while relying on third-party companies, such as Atari and Midway Manufacturing to publish them internationally under their own brands.\\
\textit{to-be-explained}:  “Inception” - “released\_in\_region” - “Denmark” \\
\textit{explanation}: “Bird” - “released\_in\_region” - “Denmark” \\

Both movies were distributed by Warners Bros. It is reasonable that a company usually distributed the movies in the same country. 

These and other identified cases show that the instances selected by GR are not only the most influential triples but also reasonable explanations for users.

\begin{figure}[!h]
	\centering
	\begin{subfigure}{.23\textwidth}
		\centering
		\includegraphics[width=1.0\linewidth]{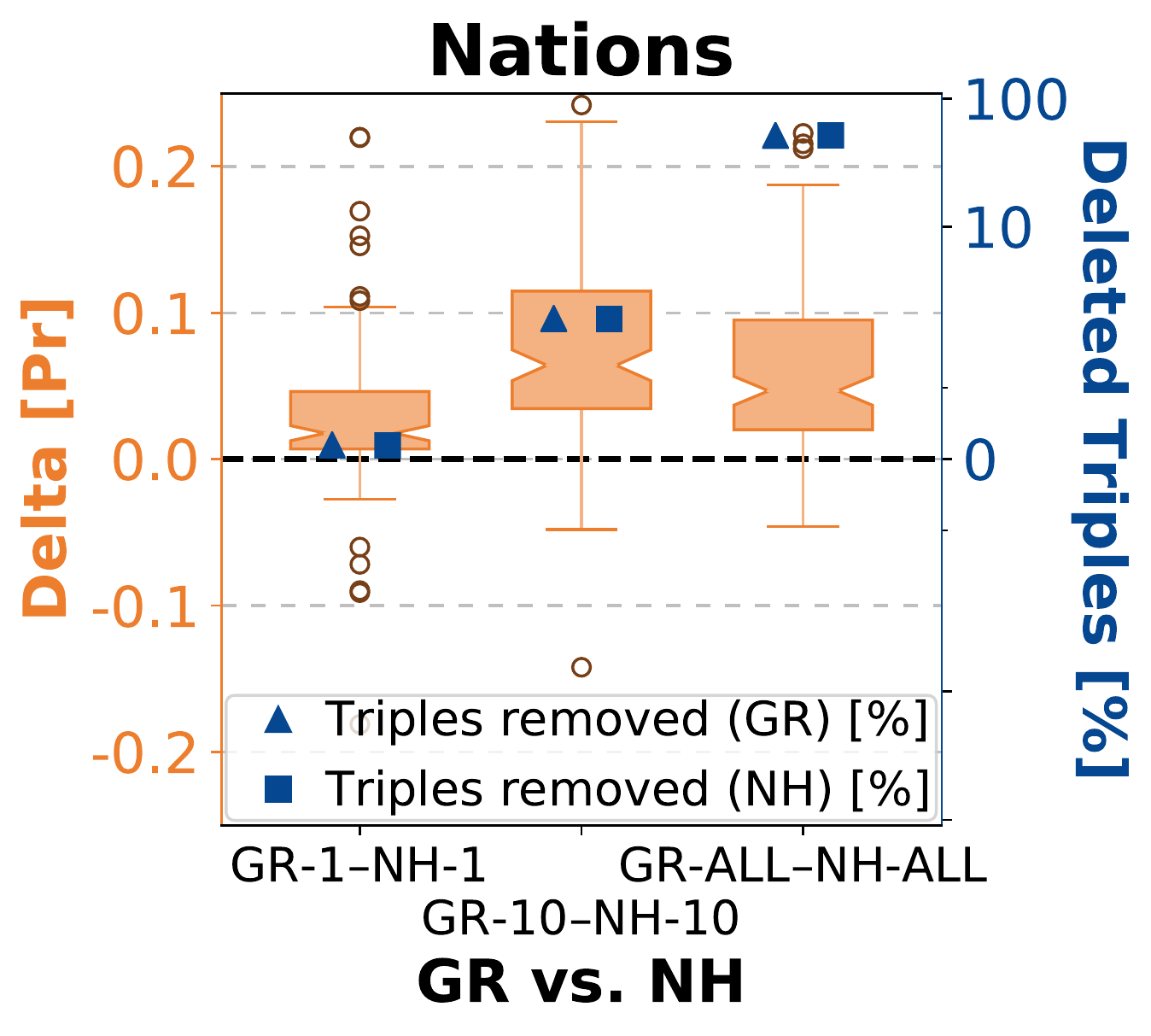}
	\end{subfigure}    
	\hspace{0.1em}
	\begin{subfigure}{.23\textwidth}
		\centering
		\includegraphics[width=1.0\linewidth]{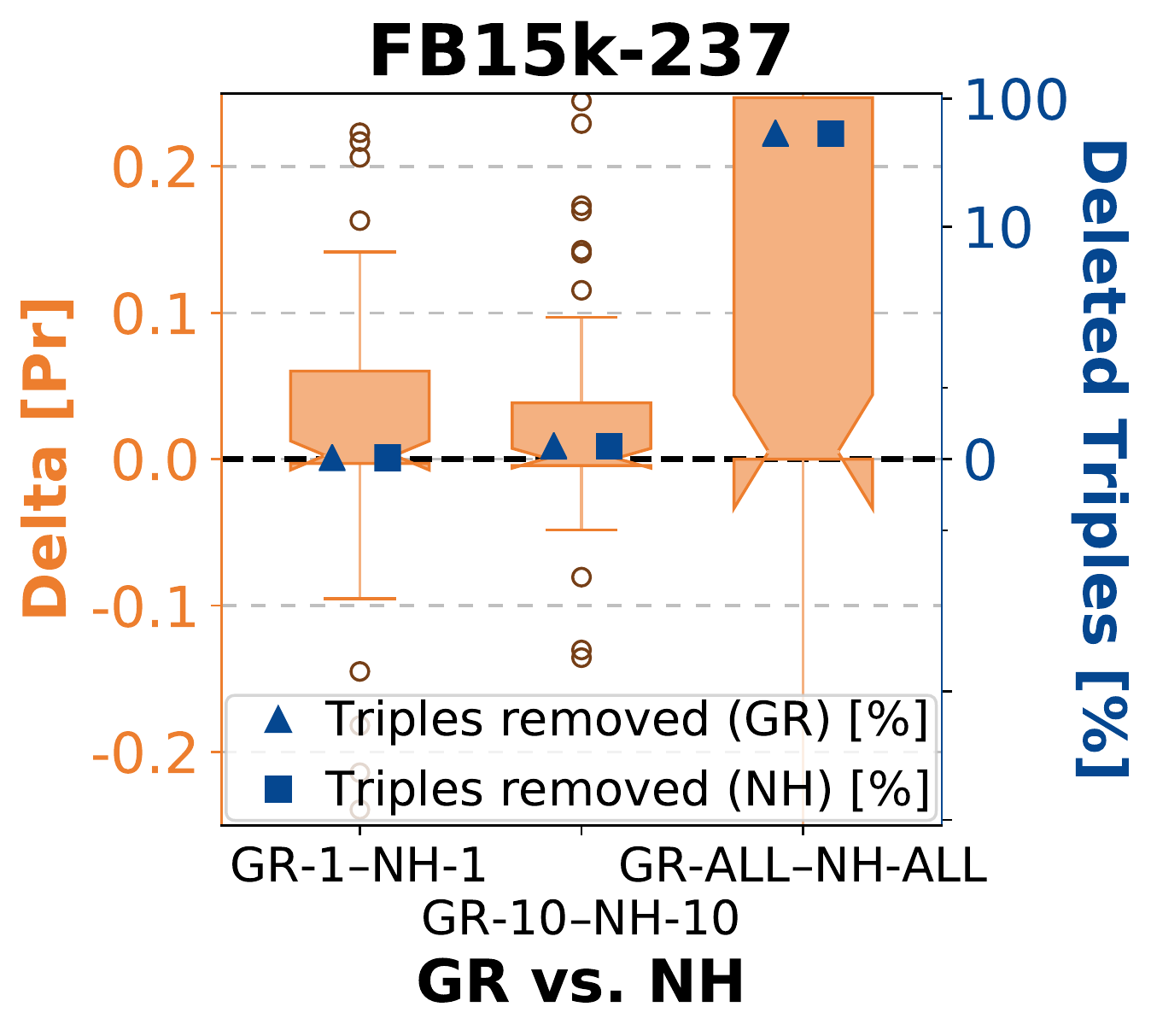}
	\end{subfigure}
	\hspace{0.1em}
	\begin{subfigure}{.23\textwidth}
		\centering
		\includegraphics[width=1.0\linewidth]{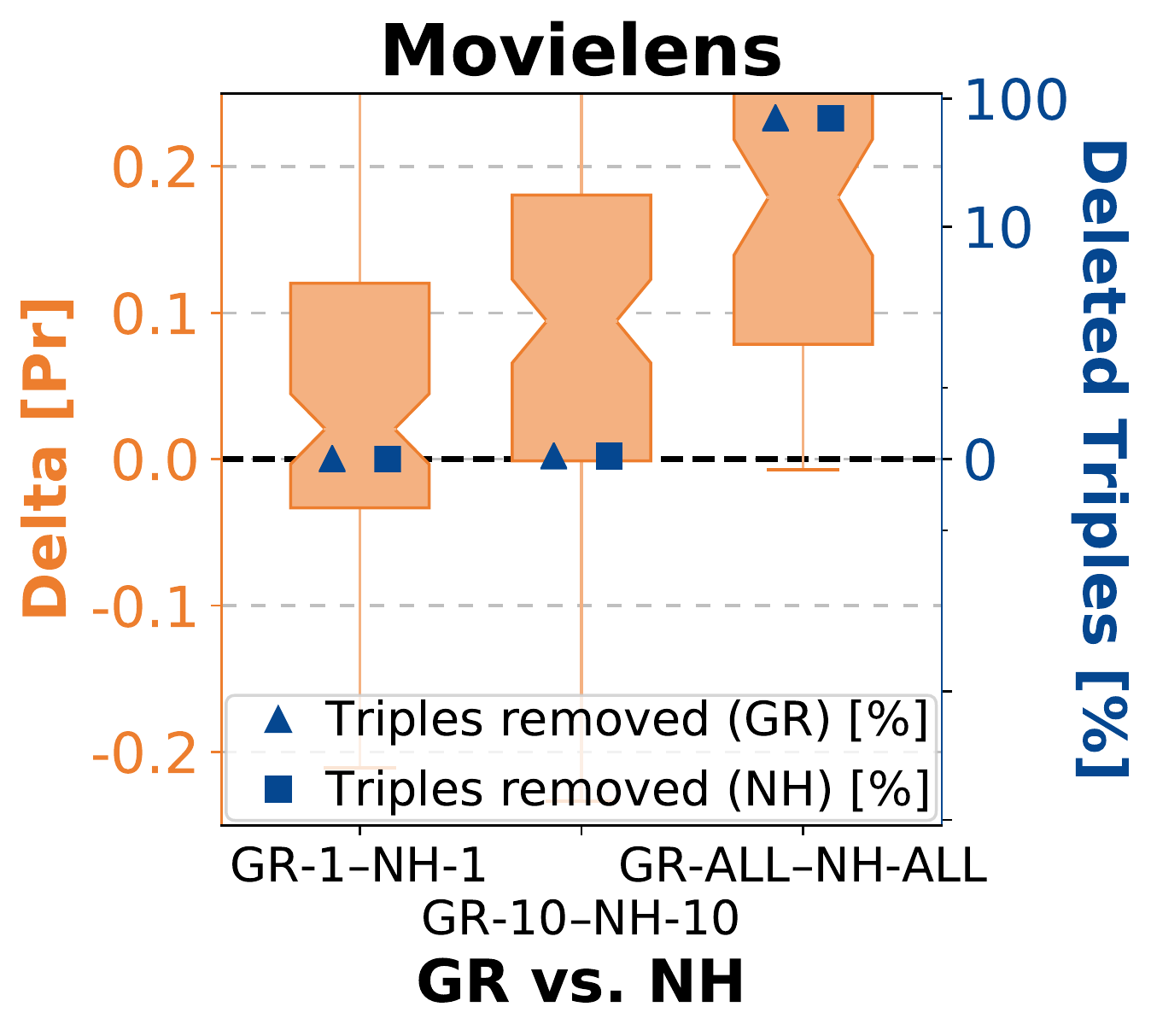}
	\end{subfigure}
	\caption{\label{fig:boxplot}The boxplots illustrate the difference between the change in probability of a test triple caused by removing the set of triples $\mathcal{S}$ selected by, respectively, GR and the baselines NH. The value is larger than zero, when GR selects triples that change the probability more after retraining than those selected by the baselines. The boxplots also depict the standard deviations and the average number of deleted triples (on the right y-axis).}
\end{figure}
\end{appendices}

\end{document}